\documentclass[twoside]{article}

\usepackage[preprint]{aistats2027}
\usepackage{notations}
\usepackage[round]{natbib}

\begin{document}

%

%

\twocolumn[

\aistatstitle{Expected Sample Complexity in Multi-Armed Bandits}

\aistatsauthor{ Nadav Sukenik \And Nadav Merlis}

\aistatsaddress{ Technion -- Institute of Technology \And  Technion -- Institute of Technology} ]

\begin{abstract}
Sample complexity is a widely used metric in sequential decision-making problems, 
defined as the number of suboptimal decisions during the interaction between the agent and an environment.  
We study the sample complexity of stochastic multi-armed bandit problems and introduce the \emph{expected sample complexity} performance measure,
analyzing it in a novel framework called approximately correct in expectation (ACE). 
We show that ACE guarantees imply almost sure convergence to the optimal expected reward, 
in contrast to high-probability guarantees found in other frameworks, and also show how to convert ACE guarantees into explicit expected regret bounds. 
We further show that, in contrast to existing measures, deterministic algorithms cannot obtain favorable ACE bounds, 
and analyze stochastic algorithms in two settings: when the allowed suboptimality level $\epsilon$ is known to the algorithm and when it is unknown. 
In the former, we devise an explore-then-$\epsilon$-greedy algorithm, and in the latter, we analyze the expected sample complexity of Thompson sampling.
Finally, we establish nearly matching lower bounds for both settings, showing that the algorithms are tight in $\epsilon$ and proving a performance separation between the two regimes.
\end{abstract}

\section{Introduction}\label{sec:intro}
The stochastic multi-armed bandit (MAB) problem \citep{BanditAlgsTor2017} is an important sequential decision making problem,
in which an agent picks an action from a set of possible arms and observes a stochastic reward for the chosen arm. 
A common measure in both MABs and the more general reinforcement learning  \citep[RL,][]{sutton1998reinforcement} setting is the expected regret,
a measure of the difference between the performance of an optimal oracle and our algorithm when interacting with the environment over a finite horizon \citep{BanditAlgsTor2017}.

Another performance metric that is mostly considered in the RL setting is the sample complexity \citep{kakade2003sample}, which measures how many $\epsilon$-suboptimal actions are played during the interaction for a given suboptimality level $\epsilon$. Algorithms for which the sample complexity is bounded with high probability are often called PAC-MDP \citep{kakade2003sample}. 

This framework has two notable disadvantages: 
(1) algorithms for sample complexity optimization are tailored to a specific suboptimality level, which leads to high regret or even infinite sample complexity for smaller suboptimality levels \citep{dann2018unifyingpacregretuniform}; 
and (2) they are tailored to a specific probability of success, which implies a constant probability that catastrophic failure may occur, in the form of infinite suboptimal rounds. While previous works extended the setting to alleviate the first issue \citep{dann2018unifyingpacregretuniform,liu2024uniformlastiterateguaranteebandits}, the second issue is largely unaddressed.

We tackle both weaknesses by introducing a novel framework for analyzing the sample complexity called \textbf{a}pproximately \textbf{c}orrect in \textbf{e}xpectation (\textbf{ACE}). 
In this framework, we measure the expected sample complexity: the expected number of rounds in which the policy was expected to act $\epsilon$-suboptimally over an infinite horizon of interaction.

We show that algorithms that have uniformly bounded expected sample complexity -- a property we call \textbf{U}niformly-\textbf{ACE} (\textbf{U-ACE}) -- 
have almost-sure convergence to the optimal expected reward. This comes in stark contrast to the fixed-probability guarantees of other sample complexity frameworks. Moreover, we show that explicit expected regret bounds can be derived from U-ACE guarantees.

Next, we show that deterministic algorithms cannot have a finite expected sample complexity on all instances of the Bernoulli bandit class, which leads us to analyze the expected sample complexity of randomized algorithms. We do so in two regimes: when the suboptimality level $\epsilon$ is known and when it is unknown. 
In the latter setting, we analyze Thompson sampling \citep{og_thompson}, a well-known MAB algorithm popular for its  strong empirical performance; for the former setting, we suggest a successive-elimination-$\epsilon$-greedy hybrid algorithm. We complement the expected sample complexity upper bounds with nearly matching lower bounds, showing that the guarantees possible in the two regimes are inherently different. To our knowledge, this is the first such separation result in the sample complexity literature.

\section{Related Work}
The stochastic multi-armed bandits problem
\citep{BanditAlgsTor2017, IntroToBanditsSlivkins2019} 
is one of the most widely studied problems in sequential decision making, with applications in the fields of advertising, recommendation systems, finance and more \citep{MAB_in_RS,bouneffouf2019surveypracticalapplicationsmultiarmed}.  

A common performance measure in the MAB literature is the expected regret, which is the cumulative difference between the reward of the best action (arm) and the agent's action \citep{BanditAlgsTor2017}. Much of the literature focuses on the design of algorithms that minimize the expected regret \citep{auer2002finite,BanditAlgsTor2017}. An example of such an algorithm is Thompson sampling (TS, \cite{og_thompson, kaufmann2012thompson}), a randomized algorithm inspired by the Bayesian statistics framework, which grew in popularity due to its demonstrated strong empirical performance 
\citep{TS_empirical_performance}. 

Another common performance measure is the sample complexity of exploration \citep{kakade2003sample,domingues2020episodicreinforcementlearningfinite}, referred to as sample complexity, which is defined as the number of $\epsilon$-suboptimal interaction rounds of the agent over an infinite horizon.
The sample complexity is usually analyzed in the probably approximately correct (PAC) framework, and the setting is commonly referred to as PAC-MDP. We also refer to this setting as PAC, and emphasize that it should not be conflated with the best arm identification \citep{BAI_MAB} setting, where an algorithm must identify a near-optimal arm with high probability.

Although the sample complexity has been formally defined for randomized policies \citep{domingues2020episodicreinforcementlearningfinite}, most works on the metric have largely focused on deterministic algorithms, with no obvious way to generalize their results to randomized policies.

This constitutes a major literature gap, since many widely used algorithms, such as Thompson Sampling \citep{og_thompson},  proximal policy optimization \citep[PPO,][]{schulman2017proximalpolicyoptimizationalgorithms}, trust region policy optimization \citep[TRPO,][]{schulman2017trustregionpolicyoptimization}, and many more \citep{lehmann2024definitiveguidepolicygradients, Policy_gradient_methods} are stochastic. 
In contrast to previous works, our framework makes the distinction between randomized and deterministic algorithms and their possible guarantees.

There are many extensions to the standard PAC framework. 
A notable example is Uniform-PAC \citep[U-PAC,][]{dann2018unifyingpacregretuniform}, 
which measures PAC uniformly across all suboptimality levels $\epsilon$, resulting in stronger guarantees compared to regular PAC. 
Another noteworthy extension is uniform last-iterate \citep[ULI,][]{liu2024uniformlastiterateguaranteebandits}, 
which instead bounds the instantaneous suboptimality gap of an algorithm with high probability, also giving rise to stronger guarantees than the vanilla PAC framework.
We define and discuss these frameworks in detail in \Cref{ssec:more_sample_complexity_extansions}. In particular, we explain how both extensions suffer from the same constant probability of potentially catastrophic failure mentioned in \Cref{sec:intro}.

Another related metric to the sample complexity is the lenient regret \citep{LenientRegretMerlis2021},
which combines the notions of regret and sample complexity into a single performance measure; we also discuss the relation between ACE and lenient regret in \Cref{subsec:regret_lenient_sample}. 
\section{Setting}
\subsection{Multi-Armed Bandits}\label{sec:MAB_setting}
We consider the stochastic multi-armed bandit problem
with $K$ arms, where each arm $a \in \brs*{K} \triangleq \brc*{1, 2, \ldots, K}$ is associated with a distribution $\nu_a$.
Let ${\nu} = \brc*{\nu_a}_{a=1}^{K}$ denote the set of all arm distributions.
At each round $t \in \N$, an agent picks a single arm to pull $A_t$ and observes a reward from only that arm. 
When pulling arm $a$ for the $n^{th}$ time, the agent observes a reward $X_{a,n} \sim \nu_a$ bounded in $\brs*{0, 1}$, 
and of expectation $\mu_a \triangleq \E\brs*{X_{a,n}}$.
Denote the expectation of an optimal arm by $\mu^* = \max_{a} \mu_a$, and the suboptimality gap of an arm $a$ by $\Delta_a \triangleq \mu^* - \mu_a$.
Additionally, we denote $\Delta_{\min} \triangleq \min_{a : \Delta_a > 0}\Delta_a$, and $\Delta_{\max} \triangleq \max_{a}\Delta_a$.

The decisions of the agent are defined by a sequence of policies $\pi \triangleq \br*{\pi_t}_{t=1}^{\infty}$, 
where each policy $\pi_t$ is a mapping from histories $\F_{t-1}$ to distributions over the arms $\Delta \br*{\brs*{K}}$. 
At round $t$, the agent samples an arm based on the current history $A_t \sim \pi_t \br*{\cdot \mid \F_{t-1}}$,
and observes a reward $X_{t} = X_{A_t, n_{t}(A_t) + 1}$, 
where $n_t \br*{a} = \sum_{i=1}^{t-1} {\Ind{ A_i = a}}$ is the number of times an arm $a$ was sampled by the agent up to round $t - 1$.

Denote the mean of arm $a$ using the first $n$ samples by
$\hat{\mu}_{a, n} = \frac{1}{n} \sum_{k=1}^{n}{X_{a,k}}$ 
and set $\hat{\mu}_{a, 0} \triangleq 0$.
Define the empirical mean up to round $t$ to be $\hat{\mu}_t \br*{a} = \hat{\mu}_{a, {n_t \br*{a}}}$.

The history $\F_{t-1}$ contains all the actions and rewards observed before round $t$, 
but not the internal randomization of the algorithm. 
We refer to \Cref{sec:interaction-filtration} for a formal definition and a detailed explanation.

We say that policy ${\pi_t}$ is deterministic if for every realization of history $\F_{t-1}$,
there exists an arm $a \in \brs*{K}$ such that $\pi_t \br*{a \mid \F_{t-1}} = 1$.
Any policy that is not deterministic is called stochastic. 
Unless stated otherwise, we let $\pi_t(a) \triangleq \pi_t(a\vert \F_{t-1})$ for brevity. 
Finally, we say that policy $\pi^*$ is optimal if $\sum_{a: \Delta_a = 0} \pi^*(a) = 1$.

\subsection{Expected Regret, Lenient Regret and Sample Complexity}\label{subsec:regret_lenient_sample}
The most common performance measure in the MAB literature is the expected regret \citep{BanditAlgsTor2017}, which measures the cumulative difference in reward between the best possible agent and an agent acting according to policy $\pi$ when interacting with instance $\nu$. The expected regret is defined to be
$$R_{T}\br*{\pi,\nu} = \E\brs*{\sum_{t=1}^{T}{\Delta_t}},$$
where $\Delta_t = \Delta_{A_t}$, and the expectation is taken over the interaction of the agent with the environment. When $\nu$ is clear from context, we write $R_{T}(\pi) = R_T(\pi, \nu)$.

Another common performance measure is the sample complexity \citep{kakade2003sample, domingues2020episodicreinforcementlearningfinite}. It measures how many suboptimal decisions an agent was expected to perform when using the policy sequence $\pi = \br*{\pi_t}_{t=1}^{\infty}$ over an infinite horizon of interaction with instance $\nu$.
Formally, we define the policy gap at round $t$ to be $\Delta_{\pi_t} \triangleq {\sum_{a=1}^{K}{\pi_{t}\br*{a}\Delta_a}}$,
and define the sample complexity to be
\begin{equation}\label{eq:sample_complexity}
    N_{\epsilon} \br*{\pi, \nu} = \sum_{t=1}^{\infty}{\Ind{\Delta_{\pi_t}>\epsilon}}
\end{equation}
for some suboptimality level $\epsilon > 0$. The notion of the policy gap arises naturally since it coincides with the realized gap $\Delta_{A_t}$ for deterministic algorithms, and measures the suboptimality of our decision rule itself -- the policy. Though it is equally meaningful to measure the suboptimality of the realized actions, we show that working with the policy gap has a clear advantage -- it facilitates the derivation of many favorable guarantees, including almost sure convergence, expected regret bounds and PAC guarantees.

We note that this definition of the sample complexity is closely related to the lenient regret metric \citep{LenientRegretMerlis2021}, which is defined as
\begin{align*}
  R^{\epsilon}_{T} \br*{\pi, \nu} = \sum_{t=1}^{T}{\Ind{\Delta_t>\epsilon}}.  
\end{align*}
The notable difference between the two is that sample complexity measures the expected suboptimality of the policy and is mostly analyzed in the PAC setting,
while lenient regret measures the suboptimality of the realized actions and is analyzed in expectation. 

To relate the policy gap and the suboptimality gap, we show the following:
\begin{restatable}{claim}{expectedRegretEquivalentForPolicyGap}\label{thm:expected_regret_equivalent_for_policy_gap}
It holds that $\E\brs*{\sum_{t=1}^{T}{\Delta_t}} = \E\brs*{\sum_{t=1}^{T}{\Delta_{\pi_t}}}.$
\end{restatable}
The proof is given in \Cref{appendix:expected_regret_equivalent_for_policy_gap}. While this claim implies that optimizing the expected cumulative policy gap and the expected regret are equivalent, we note that this is not necessarily true for the infinite-horizon lenient regret and the sample complexity.

\subsection{PAC, Uniform-PAC and Uniform Last-Iterate}\label{ssec:more_sample_complexity_extansions}
For the following definitions, we follow the conventions used in \cite{domingues2020episodicreinforcementlearningfinite} to formally extend the sample complexity framework to randomized algorithms. In doing so, we adopt the notion of the policy gap for the frameworks, and note that all the definitions align with their original statements for deterministic algorithms. We remark that while \cite{domingues2020episodicreinforcementlearningfinite} defined the sample complexity to accommodate randomized policies, their sample complexity results only apply to deterministic ones.
The most common framework in the literature of sample complexity is the PAC framework: 

\begin{definition}[\cite{kakade2003sample}]
    A policy $\pi$ is called $\br*{\epsilon, \delta}-$\textbf{PAC} with a finite sample complexity bound $m_{PAC} \br*{\epsilon, \delta}$ if
        $\Pr \br*{N_{\epsilon} \br*{\pi, \nu} > m_{PAC} \br*{\epsilon, \delta}} \le \delta.$
\end{definition}
Algorithms in this setting are designed to minimize the sample complexity bound for a given $\br*{\epsilon,\delta}$ pair, which leads to two notable disadvantages:
\begin{enumerate}
    \item Algorithms are tailored to the given suboptimality level $\epsilon$, resulting in suboptimal performance in terms of expected regret, high-probability regret, and even $\br*{\epsilon',\delta}$-PAC for $\epsilon' < \epsilon$ \citep{dann2018unifyingpacregretuniform}. In practice, we favor algorithms that can be agnostic or robust to the parameters $\epsilon,\delta$.
    \item Algorithms are designed for the given confidence $\delta$, and as such can fail with constant probability, resulting in infinite sample complexity, a highly undesirable trait.
\end{enumerate}
Extensions have been made to the existing PAC framework to mitigate the first problem, but to our knowledge, none have made progress in the latter issue. 
One such framework is Uniform-PAC:
\begin{definition}[\citep{dann2018unifyingpacregretuniform}]
    A policy $\pi$ is called $\delta$-\textbf{U-PAC} for $\delta > 0$ with sample complexity bound $m_{U-PAC} \br*{\epsilon, \delta}$ if 
    $$\Pr \br*{\exists \epsilon > 0 : N_{\epsilon} \br*{\pi, \nu} > m_{U-PAC} \br*{\epsilon, \delta}} \le \delta.$$
\end{definition}
It can be shown that algorithms that are $\delta$-U-PAC are also $\br*{\epsilon, \delta}-$PAC for all $\epsilon > 0$, and enjoy high-probability regret bounds \citep{dann2018unifyingpacregretuniform}, which addresses the major clauses in the first disadvantage. Unfortunately, this framework does not address the second problem, still allowing catastrophic failure to possibly occur with constant probability. 

Another noteworthy framework is the Uniform Last-Iterate (ULI) framework \citep{liu2024uniformlastiterateguaranteebandits}, a stronger notion than U-PAC that gives us sharper bounds. 
\begin{definition}[ULI]
   A policy $\pi$ is called $\delta$-\textbf{ULI} if there exists a function $m_{ULI}\br*{\delta, t} \leq\text{poly}\br*{\frac{1}{t},\log\frac{1}{\delta}}$, such that 
       $\Pr \br*{\exists t : \Delta_{\pi_t} > m_{ULI} \br*{\delta, t}} \le \delta.$
\end{definition}
While the U-PAC framework allows algorithms to play maximally suboptimal arms for a very large but finite horizon, the ULI framework instead characterizes the actual convergence rate of $\Delta_{\pi_t}$. The main theorem in \cite{liu2024uniformlastiterateguaranteebandits} states that ULI implies U-PAC, and their proof generalizes to the policy gap. 
As such, any policy which is ULI also suffers from the same disadvantages as U-PAC.

\subsection{Expected Sample Complexity and the ACE Framework}
We introduce a new sample complexity measure, the expected sample complexity, which for a given policy $\pi$ and instance $\nu$ is defined as:
\begin{align*}
    \E\brs*{N_{\epsilon} \br*{\pi, \nu}} = \E\brs*{\sum_{t=1}^{\infty}{\Ind{\Delta_{\pi_t}>\epsilon}}}.
\end{align*}
When $\nu$ is clear from context, we let ${N_{\epsilon} \br*{\pi}} \triangleq {N_{\epsilon} \br*{\pi, \nu}}$.

We also introduce a novel framework for the analysis of the expected sample complexity called the approximately correct in expectation (ACE) framework. 
We use the following definitions:
\begin{definition}
A policy $\pi$ is called $\epsilon$-\textbf{ACE} 
$\br*{\textbf{A}\text{pproximately } \textbf{C}\text{orrect } \text{in } \textbf{E}\text{xpectation}}$ on $\nu$ with sample complexity bound $m_{ACE} \br*{\epsilon, \nu} < \infty$ if: 
$$\E \brs*{N_{\epsilon} \br*{\pi, \nu}} \le m_{ACE} \br*{\epsilon, \nu}.$$
\end{definition}

\begin{definition}
    A policy $\pi$ is \textbf{U-ACE} 
(\textbf{U}niformly-\textbf{A}pproximately  \textbf{C}orrect in \textbf{E}xpectation) on $\nu$ if $\pi$ is $\epsilon$-ACE with bound $m_{U-ACE} \br*{\epsilon, \nu}$ for all $\epsilon > 0$. 
\end{definition}
\section{Properties of (U-)ACE Algorithms}
In this section, we analyze the properties of ACE and U-ACE policies, including convergence, expected regret bounds, and the necessity of stochastic policies to be U-ACE.

We begin by showing that ACE bounds can be used to prove the convergence of algorithms. In the following proposition, we show that $\epsilon$-ACE implies almost-sure $\epsilon$ boundedness of the policy gap
\begin{restatable}{proposition}{ACEImpliesASBoundedness}\label{thm:ACE_implies_AS_boundedness}
    Let $\pi = \br*{\pi_t}_{t=1}^{\infty}$ be $\epsilon$-ACE for some $\epsilon > 0$. Then $\limsup_{t \rightarrow \infty} \Delta_{\pi_t} \le \epsilon$ almost surely.
\end{restatable}
The proof of this proposition is given in \cref{appendix:ACE_implies_AS_boundedness} and relies on the Borel–Cantelli lemma. As a corollary, we get that U-ACE implies almost sure convergence of the policy gap to zero:
\begin{restatable}{corollary}{UACEImpliesASConvergence}\label{thm:U-ACE_implies_AS_convergence}
    Let $\pi = \br*{\pi_t}_{t=1}^{\infty}$ be U-ACE. Then $\lim_{t \rightarrow \infty} \Delta_{\pi_t} = 0$ almost surely.
\end{restatable}
This corollary is proved by applying \Cref{thm:ACE_implies_AS_boundedness} to a sequence $\epsilon \downarrow 0$; see \Cref{appendix:ACE_implies_AS_boundedness} for a formal proof.

This result yields a much stronger guarantee than the high-probability convergence implied by previous frameworks. 
We do remark that \Cref{thm:U-ACE_implies_AS_convergence} does not imply convergence of the sequence $\pi_t$ to an optimal policy $\pi^*$. In fact, $\pi_t$ may not converge at all; for example, if the policy flips between two distinct optimal policies. Yet, the corollary indicates that the expected reward of the sequence converges to that of an optimal policy, which we believe to be the more meaningful definition for a convergence of a policy to its optimum.

Next, we show that for algorithms with polynomial $m_{U-ACE}$, the expected regret is sublinear:  
\begin{restatable}{claim}{UACEImpliesExpectedReget}\label{thm:UACE_implies_expected_regret}
    Assume that $\pi$ is U-ACE with $m_{U-ACE}(\epsilon, \nu) \le \frac{C_\nu}{\epsilon^{\alpha}}$ for $C_\nu>0$, $\alpha > 1$. Then for all $T$, it holds that $R_T(\pi,\nu) \le \frac{\alpha}{\alpha - 1} C_{\nu}^{1/\alpha}T^{\frac{\alpha - 1}{\alpha}}$.  
    
    Moreover, for $\alpha = 1$, we obtain a logarithmic bound of $R_T(\pi,\nu) \le C_{\nu}\br*{1 + \max\brc*{0,\ln\br*{\frac{T}{C_{\nu}}}}}$.
\end{restatable}
Proof of this claim can be found in \Cref{appendix:UACE_implies_expected_regret}.
After showing the favorable properties of the framework, we end this section with an impossibility result for deterministic algorithms: such algorithms cannot be U-ACE for all bandit instances. This will later guide us to only focus on stochastic policies.
\begin{restatable}{claim}{deterministicAlgSublinearRegretAndUACE}\label{thm:deterministic_alg_sublinear_regret_and_U-ACE}
    Let $\Ecal_{Ber}^{K}$ be the class of $K$-armed Bernoulli bandits. 
    For any deterministic policy $\pi$ and $K \ge 2$, there exists $\nu \in \Ecal_{Ber}^{K}$ and $\epsilon > 0$ s.t.
        $\E\brs*{N_{\epsilon}\br*{\pi, \nu}} = \infty.$
\end{restatable}
A proof for this claim is given in \cref{appendix:deterministic_alg_sublinear_regret_and_U-ACE} and is based on the following observation: for deterministic algorithms, if $\epsilon<\Delta_{\min}$, the expected sample complexity is the number of rounds in which a suboptimal arm was played. Thus, it suffices to show that there exists an instance in which the expected number of suboptimal plays is infinite. This can be proved using known lower bounds for the bandit problem.

At first glance, \Cref{thm:deterministic_alg_sublinear_regret_and_U-ACE} seems to imply that U-ACE is not achievable; indeed, a simple application of Markov's inequality gives a PAC guarantee (albeit with polynomial dependence on $1/\delta$), since
\begin{align*}
    \Pr \br*{N_{\epsilon} \br*{\pi, \nu} > m_{PAC} \br*{\epsilon, \delta}} \le \frac{\E \brs*{N_{\epsilon} \br*{\pi}}}{m_{PAC} \br*{\epsilon, \delta}}.
\end{align*}
Thus, for $m_{PAC} \br*{\epsilon, \delta} = \frac{\E \brs*{N_{\epsilon} \br*{\pi}}}{\delta}$ we get $(\epsilon, \delta)$-PAC, while \Cref{thm:UACE_implies_expected_regret} simultaneously implies an expected regret bound. This seemingly contradicts Theorem 1 of \cite{dann2018unifyingpacregretuniform}, which states that no algorithm may achieve sublinear regret and finite PAC bounds for arbitrarily small $\epsilon$. Yet, a closer inspection of \citep{dann2018unifyingpacregretuniform} reveals their analysis assumes deterministic policies, and their methods are inapplicable for stochastic ones. This key observation and \Cref{thm:deterministic_alg_sublinear_regret_and_U-ACE} both encourage us to turn our attention to stochastic algorithms in our pursuit of U-ACE.

In the following section, we show that stochastic algorithms can enjoy both near-optimal expected regret, U-ACE and a finite $\br*{\epsilon, \delta}$-PAC bound for all $\epsilon > 0$, 
a surprising result considering most of the literature on sample complexity studies deterministic algorithms, dating back to the seminal work of \cite{kakade2003sample}.
\section{Algorithms for the ACE Framework}
We now analyze the expected sample complexity of two classes of algorithms: algorithms that receive $\epsilon$ as input (non-uniform algorithms), and algorithms that do not (uniform algorithms). 
We show that the dependence of the expected sample complexity on $\epsilon$ is inherently different for the different classes; to our knowledge, this is the first separation result of this type in the sample complexity setting. 

\subsection{Non-Uniform Algorithms}\label{sec:non_uniform_algs}
This setting is the most similar to the vanilla PAC framework; the algorithm receives $\epsilon$ as input, and its goal is to achieve $\epsilon$-ACE with minimal $m_{ACE}\br*{\epsilon, \nu}$. While algorithms designed for this setting can only support suboptimality levels equal or larger than theirs (in contrast to uniform algorithms), we will show that this regime leads to better ACE bounds. 

Before presenting our algorithm for this regime, we explain our design methodology; one can surmise that algorithms for this setting must explore infinitely often to stay $\epsilon$-optimal over an infinite horizon in expectation. We turn our attention to $\epsilon$-greedy \citep{auer2002finite}, a simple algorithm that performs biased coin tosses to decide whether to explore or exploit the environment at each round. Formally, for a given sequence of biases $\br*{\epsilon_t}_t$, the $\epsilon$-greedy policy is defined as
\begin{align*}
    \pi^{\epsilon-greedy}_t \br*{a \mid \F_{t-1}} = \frac{\epsilon_t}{K} + \br*{1 - \epsilon_t}\Ind{a = g_t},
\end{align*}
where $g_t \in \argmax_{a'}\hat{\mu}_t \br*{a'}$ is a maximizer of the empirical means in round $t$.
A simple approach would then be to run $\epsilon$-greedy with $\epsilon_t \triangleq \epsilon /2$. Intuitively, after enough rounds, the greedy action is optimal w.h.p., and so the policy gap is small since our exploration probability is also small by our choice of bias.

While such a policy will be $\epsilon$-ACE, it will be sample inefficient; the probability of exploration is too small, making the number of rounds needed to detect an optimal arm w.h.p. very large, incurring high expected sample complexity even in `easy' instances.

To balance the need for fast exploration at the beginning with more careful exploitation in later rounds, one can start with a more aggressive exploration algorithm and then transition to a balanced exploration-exploitation algorithm in later rounds. 
With this motivation in mind, we propose the Successive Elimination Then $\epsilon$-Greedy (SET$\epsilon$-G) algorithm (depicted in \cref{alg:known_epsilon_SE}), 
and analyze its expected sample complexity. 
The algorithm has two phases: an initial successive elimination phase for exploration and pruning of bad arms, and an $\epsilon$-greedy phase to exploit discovered good arms but keep exploring in case of misidentification. The algorithm enjoys the following ACE bound (see \cref{appendix:known_epsilon_upper_bound} for the full proof):

\begin{algorithm}[!t]
\caption{Successive Elimination Then $\epsilon$-Greedy}
\label{alg:known_epsilon_SE}
\begin{algorithmic}[1]
\STATE \textbf{Input:} $\epsilon > 0$
\STATE Set $\mathcal A_1 = [K]$, $q = \frac{\epsilon^2}{16K}$
\item[] \textcolor{gray}{\textbf{\# Phase I}} 
\FOR{$n=1,2,\ldots,n_\epsilon \triangleq \min\brc*{n \ge 1 :\beta_n < \frac{\epsilon}{32}}$}
    \STATE Set $\beta_n=\sqrt{\frac{1}{2n}\ln\br*{\frac{64K^2n^2}{\epsilon^2}}}.$
    \FOR{every $a \in \mathcal A_n$}
        \STATE Pull arm $a$ once
    \ENDFOR
    \STATE Let $\hat{\mu}_{a,n}$ denote the empirical mean of the first
    $n$ samples of every $a \in \mathcal A_n$
    \STATE Let {\small$\mathcal{A}_{n+1}
        \!=\!\brc*{a \!\in\! \mathcal A_n :\max_{b \in \mathcal A_n}\hat{\mu}_{b,n}-\hat{\mu}_{a,n}\le 6\beta_n}$}
    \IF{$|\mathcal A_{n+1}|=1$}
        \STATE Terminate Phase I
    \ENDIF
\ENDFOR
\item[] \textcolor{gray}{\textbf{\# Phase II}}
\STATE Initialize $\hat{\mu}_1(a), S_1(a), n_1(a)$ for all $a\in\brs*{K}$ using the samples from phase I
\FOR{remaining rounds $s = 1, 2, \ldots$}{
    \STATE Sample $B_s \sim Ber \br*{\min\brc*{{\epsilon}/{4}, 1}}$
    \IF{$B_s = 1$}{
        \STATE Play $a_s \sim uni \br*{\brs*{K}}$ and observe reward $X_s$
    }
    \ELSE{
        \STATE Play $a_s \in \argmax_{a'} \hat{\mu}_s (a')$ (break ties by smallest arm index) and observe reward $X_s$
    }
    \ENDIF
    \STATE Update $n_{s+1}(a_s) = n_{s}(a_s) + 1$,\\ $S_{s+1}(a_s) = S_s (a_s) + X_s,\ \hat{\mu}_{s+1}(a_s) = \frac{S_{s+1}(a_s)}{n_{s+1}(a_s)}$
    \STATE For all $a \ne a_s$, set 
    $n_{s+1}(a) = n_{s}(a),$\\ $S_{s+1}(a) = S_s (a),\ \hat{\mu}_{s+1}(a) = \hat{\mu}_{s}(a)$
}\ENDFOR
\end{algorithmic}
\end{algorithm}

\begin{restatable}{theorem}{KnownEpsilonUpperBound}\label{thm:known_epsilon_upper_bound}
Let $\epsilon > 0$, and let $\pi^{\text{SET$\epsilon$-G}}$ denote the policy induced by
\cref{alg:known_epsilon_SE}. Then for all $\nu$:
{\small\begin{align*}
    \E\brs*{N_\epsilon\br*{\pi^{\text{SET$\epsilon$-G}},\nu}} = \Ocal\br*{\sum_{\Delta_a > \frac{3\epsilon}{4}} \frac{1}{\Delta_a^2}\ln\br*{\frac{K}{\epsilon^2}}}.
\end{align*}}
\end{restatable}
\begin{proofsketch}
    We analyze both phases separately. We define a concentration event for the means over the entire horizon, such that the probability it is violated is $\Ocal(\epsilon^2)$.
    For the initial successive elimination phase, we show that under the good event, 
    suboptimal arms with $\Delta_a > \epsilon$ are eliminated after $\Olog\br*{{1}/{\Delta_a^2}}$ rounds. Since this phase is deterministic, the expected sample complexity is exactly the number of rounds we play arms with gap larger than $\epsilon$: bounded by $\Olog\br*{\sum_{a:\Delta_a > \epsilon}\frac{1}{\Delta_a^2}}$.

    For the $\epsilon$-greedy phase, we show that the probability of an arm $a$ with gap $\Delta_a > \frac{3\epsilon}{4}$ being the empirical maximizer is uniformly bounded across all rounds by $\Ocal(\epsilon^2)$, and also by an exponentially decreasing function of the rounds of order $e^{-\epsilon\Delta_a^2s}$, where $s$ is the current round of the phase, giving us $\Pr\br*{g_s = a} = \Ocal\br*{\min\brc*{\epsilon^2, e^{-\epsilon\Delta_a^2s}}}$. 
    In \Cref{lem:min_constant_exponential_sum}, we show that  $\Ocal\br*{\sum_{j=0}^{\infty}\min\brc*{\epsilon^2, e^{-\epsilon\Delta_a^2s}}} = \Olog\br*{\frac{1}{\Delta_a^2}}$, 
    and summing over the bad arms gives us a bound of $\Olog\br*{\sum_{\Delta_a > \frac{3\epsilon}{4}}\frac{1}{\Delta_a^2}}$ for the $\epsilon$-Greedy phase. 
    Combining the bounds for both phases yields the result.
\end{proofsketch}

We complement the ACE bound for SET$\epsilon$-G with an almost-matching lower bound, showing that \Cref{alg:known_epsilon_SE} is near-optimal in the non-uniform setting, up to logarithmic factors in $K$ and $1/\epsilon$.
\begin{restatable}{proposition}{MinMaxLowerBound}\label{thm:lower_bound_delta_and_epsilon}
Let $\delta, \epsilon > 0 $ such that $\max\brc*{\delta, 3\epsilon} \le 1/8$, and let $K \ge 2$. 
Define the bandit class $\Ecal_{\delta} = \brc*{\nu \in \Ecal_{Ber}^{K} : \Delta_{\min} \br*{\nu} \ge \delta}$.
Then for every policy $\pi$, there exists $\nu \in \Ecal_{\delta}$ such that: 
$$\E \brs*{N_{\epsilon} \br*{\pi, \nu}} 
= \Omega \br*{\frac{K}{\max \brc*{\delta ,\epsilon}^2}}.$$
\end{restatable}
\begin{proofsketch}(see \cref{appendix:lower_bounds} for the full proof)
    Our lower bound design methodology closely follows Theorem 15.2 in \cite{BanditAlgsTor2017}.
    We let $\Delta = \max\brc*{\delta, 3\epsilon}$, fix some $T \in \N$ to be chosen later, and create two instances, the first one $\nu$ is defined as
    \begin{align*}
        \mu_{\nu}(a) 
        = \begin{cases}
        1/2 + \Delta & a = 1 \\
        1/2 & a \ne 1
        \end{cases},
    \end{align*} 
    and the second instance $\nu'$ is such that the arm that is played least on $\nu$ in expectation is set to $1/2 + 2\Delta$, with the rest of the arms remaining unchanged.
    We also let $\E\brs*{N_\epsilon \br*{\pi, \nu, T}} = \E\brs*{\sum_{t=1}^{T} \Ind{\Delta_{\pi_t} > \epsilon}}$ be the $T$-expected sample complexity. 
    Finally, we define $A$ to be the event that the number of times $\pi_t(1)$ was the dominant probability over the horizon is smaller than $T/2$. This implies that the $T$-expected sample complexity in $\nu$ is at least $\Pr_{\nu}(A)\cdot \frac{T}{2}$, and the $T$-expected sample complexity in $\nu'$ is at least $\Pr_{\nu'}\br*{\overline{A}}\cdot \frac{T}{2}$. Apply the standard Bretagnolle-Huber approach (Theorem 14.2 in \cite{BanditAlgsTor2017}) gives us
    \begin{align*}
        \max \brc*{\E_{\nu} \brs*{N_{\epsilon} \br*{\pi, \nu, T}},\E_{\nu'} \brs*{N_{\epsilon} \br*{\pi, \nu', T}}} 
        = \Omega\br*{Te^{-\frac{T\Delta^2}{K}}}
    \end{align*}
    To get the desired result, we note that $\E\brs*{N_\epsilon \br*{\pi, \nu}} \ge \E\brs*{N_\epsilon \br*{\pi, \nu, T}}$ for all $T$, and choose $T^*$ that maximizes the above lower bound.
\end{proofsketch}
Notably, this lower bound holds even for algorithms that obtain both $\epsilon$ and $\delta$ as input, whereas our algorithm assumes access only to $\epsilon$. Moreover, choosing $\delta = 3\epsilon$, we immediately obtain the following lower bound for when only $\epsilon$ is known:
\begin{corollary}\label{thm:lower_bound_epsilon}
Let $\epsilon \in (0,1/24]$ be given. Then for every $\pi$, there exists $\nu \in \Ecal_{Ber}^{K}$ such that
\begin{align*}
    \E \brs*{N_{\epsilon} \br*{\pi, \nu}} 
    = \Omega \br*{{K}/{{\epsilon^2}}}.
\end{align*}
\end{corollary}
\Cref{thm:lower_bound_delta_and_epsilon} highlights a delicate tradeoff when $\epsilon$ is known. If $\epsilon$ is large compared to the instance gaps, one should explore only to $\epsilon$-accuracy. On the other hand, if $\epsilon$ is very small, then it is better to aggressively explore to identify optimal arms.

\subsection{Uniform Algorithms}
In contrast to the non-uniform regime, algorithms that do not optimize for a given suboptimality level can enjoy the full benefit of the U-ACE framework, with expected regret guarantees (\Cref{thm:UACE_implies_expected_regret}) and almost sure convergence (\Cref{thm:U-ACE_implies_AS_convergence}). We show that these benefits come at the cost of worse expected sample complexity bounds, compared to non-uniform algorithms.

Noting that this regime is strictly harder than the non-uniform regime, the previous lower bounds also hold for uniform algorithms.
Here, we present more refined lower bounds for uniform algorithms, showing the setting has an intrinsic polynomial dependence on $1/\epsilon$.
 Recall that a policy $\pi$ is consistent over bandit class $\Ecal$ if for all $\nu \in \Ecal$ and $p > 0$:
    $\lim_{T \rightarrow \infty} \frac{R_T \br*{\pi, \nu}}{T^{p}} = 0.$
\begin{restatable}{proposition}{UACEAndConsistentLowerBound}\label{thm:U-ACE_and_consistent_lower_bound}
    Let $\pi$ be U-ACE and consistent over $\Ecal_{Ber}^{K}$. Then for any fixed instance $\nu \in \Ecal_{Ber}^{K}$ with means in $[1/4, 3/4]$, it holds that 
    \begin{align*}
        \limsup_{\epsilon\downarrow 0} \epsilon \cdot \E\brs*{N_{\epsilon}\br*{\pi, \nu}} \ge \sum_{a:\Delta_a > 0}\frac{3}{8\Delta_a}.
    \end{align*}
\end{restatable}
This lower bound shows that the instance-dependent bound for consistent U-ACE algorithms is larger than the proven upper bound for the non-uniform SET$\epsilon$-G algorithm.
The following proposition shows that a polynomial dependence on $1/\epsilon$ is still present even if we lift the consistency assumption:
\begin{restatable}{proposition}{UACEWithoutConsistencyLowerBound}\label{thm:U-ACE_without_consistency_lower_bound}
    Let $\pi$ be U-ACE over $\Ecal_{Ber}^{2}$. Then there exists an instance $\nu \in \Ecal_{Ber}^{2}$ such that for any $\alpha \in (0,1)$:
    \begin{align*}
       \limsup_{\epsilon\downarrow 0} \epsilon^{\alpha} \cdot \E\brs*{N_{\epsilon}(\pi,\nu)} = \infty.
    \end{align*} 
\end{restatable}
Proofs for \Cref{thm:U-ACE_and_consistent_lower_bound,thm:U-ACE_without_consistency_lower_bound} 
can be found in \Cref{appendix:U-ACE_and_consistent_lower_bound,appendix:U-ACE_without_consistency_lower_bound}, respectively.
These lower bounds show a clear separation from the non-uniform setting.

Next, we analyze the expected sample complexity of Thompson sampling, a widely used algorithm in practice known for its strong empirical performance and robustness, and show that it is U-ACE.

Thompson sampling works by maintaining a prior over the means of all arms, and at each round drawing independent samples from each prior. The arm with the largest sample is chosen, and the observed reward is then used to update the posterior of that arm. The algorithm works well in practice even when the problem is not sampled from the prior distribution, and enjoys near-optimal expected regret guarantees \citep{AgrawalGoyalTS2017}. We show the following: 
\begin{algorithm}[t]
\caption{Thompson sampling with normal priors}
\label{alg:TS_N}
\begin{algorithmic}[1]
\STATE Initialize $n_1(a) = 0, S_1(a) = 0, \hat{\mu}_1(a) = 0$ for all $a \in \brs*{K}$
\FOR{rounds $t = 1, 2, \ldots$}{
    \FOR{$a = 1, \ldots, K$}{
          \STATE Sample $\theta_t(a) \sim \Ncal \br*{\hat{\mu}_t (a), \frac{1}{1 + n_t(a)}}$   
    }\ENDFOR
    \STATE Play $a_t \in \argmax_{a'} \theta_t(a')$ (break ties arbitrarily) and observe reward $X_t$
    \STATE Update $n_{t+1}(a_t) = n_{t}(a_t) + 1, S_{t+1}(a_t) = S_t (a_t) + X_t, \hat{\mu}_{t+1}(a_t) = \frac{S_{t+1}(a_t)}{1 + n_{t+1}(a_t)}$
    \STATE For all $a \ne a_t$, set 
    $n_{t+1}(a) = n_{t}(a), S_{t+1}(a) = S_t (a), \hat{\mu}_{t+1}(a) = \hat{\mu}_{t}(a)$
}\ENDFOR
\end{algorithmic}
\end{algorithm}
\begin{restatable}{theorem}{thompsonSamplingIsUACE}\label{thm:thompson_sampling_is_U-ACE}
Let $\pi^{TS \br*{\Ncal}}$ be the policy induced by \Cref{alg:TS_N}.  
Then for all $\epsilon > 0$ and instances $\nu$:
{\small\begin{align*}
    \E \brs*{N_{\epsilon} \br*{\pi^{TS \br*{\Ncal}}, \nu}} 
= {\Ocal} \br*{ \frac{K}{\epsilon} 
\sum_{\Delta_a \ge \frac{\epsilon}{K}} \frac{1}{\Delta_a}{\ln \br*{\frac{3K \Delta_a}{\epsilon}}}}.
\end{align*}}
\end{restatable}
We defer the full proof to \cref{appendix:thompson_sampling_is_U-ACE}.
\begin{proofsketch}
    We first aim to manipulate the sample complexity measure to a form that allows using decompositions similar to \citep{AgrawalGoyalTS2017}. In particular, 
    each summand of the expected sample complexity can be decomposed as a sum over the probability of choosing each arm:
    \begin{lemmaof}{lem:esier_event_to_analyze_TSN}
        \begin{align*}
            \Pr \br*{\Delta_{\pi_t} > \epsilon} 
            \le \sum_{a : \Delta_a \ge \frac{\epsilon}{K}} 
            \Pr \br*{\pi_t(a) > \frac{\epsilon}{K \Delta_a}}.
        \end{align*}
    \end{lemmaof}
    Combined with the Fubini-Tonelli Theorem, this yields
    \begin{align*}
    \E\brs*{\sum_{t=1}^{\infty}{\Ind{\Delta_{\pi_t}>\epsilon}}}
    \le \sum_{\Delta_a \ge \frac{\epsilon}{K}}{\sum_{t=1}^{\infty}\Pr \br*{\pi_t(a) > \frac{\epsilon}{K \Delta_a}}}.
    \end{align*}
    Fixing arm $a$ with $\Delta_a \ge \frac{\epsilon}{K}$, we decompose the inner sum into 3 terms:
    {\small
    \begin{align*}
        &{\sum_{t=1}^{\infty}\Pr \br*{\pi_t(a) > \frac{\epsilon}{K \Delta_a}}} \\
        &\le \underbrace{\sum_{t=1}^{\infty}\Pr\br*{\Pr\br*{A_t = a, \overline{E_{a}^{\mu}\br*{t}} \vert \F_{t-1}} > \frac{\epsilon}{3K \Delta_a}}}_{S_1}\\
        & + \underbrace{\sum_{t=1}^{\infty}\Pr\br*{\Pr\br*{A_t = a, E_{a}^{\mu}\br*{t}, \overline{E_{a}^{\theta}\br*{t}} \vert \F_{t-1}} > \frac{\epsilon}{3K \Delta_a}}}_{S_2} \\
        & + \underbrace{\sum_{t=1}^{\infty}\Pr\br*{\Pr\br*{A_t = a, E_{a}^{\mu}\br*{t}, E_{a}^{\theta}\br*{t} \vert \F_{t-1}} > \frac{\epsilon}{3K \Delta_a}}}_{S_3},
    \end{align*}}
    where $E_a^\mu(t) \triangleq \brc*{\hat\mu_t(a)\le x_a}$ and $E_a^\theta(t) \triangleq \brc*{\theta_t(a)\le y_a}$ are the events that the empirical means and the prior samples are not overestimated, respectively. Then, we bound terms $S_1$ and $S_2$ using similar methods to \cite{AgrawalGoyalTS2017}. We remark that parts of these derivations require more careful treatment than the original proof, since the sample complexity is defined for an infinite horizon (in contrast to regret); we defer these details to the full proof.

    To bound $S_3$, the approach from \cite{AgrawalGoyalTS2017} does not trivially generalize to the infinite horizon setting. Moreover, their proof relies on the uniqueness of the optimal arm, which can be assumed w.l.o.g. for regret minimization. Technically, it was proven by showing that adding an optimal arm decreases the regret, since whenever TS plays it, this arm incurs no regret. However, this is no longer the case in our setting; in \Cref{sec:optimal_arms_TS}, we provide a numerical example in which adding an optimal arm strictly increases the finite-time expected sample complexity. To deal with this, we fix some optimal arm $a^*$ and define
    \begin{align*}
        p_{a,t} \triangleq \Pr\br*{\theta_{t}(a^*)>y_a\vert \F_{t-1}},
        \qquad
        q_{a,t}\triangleq\frac{1-p_{a,t}}{p_{a,t}}.
    \end{align*}
    Then, similarly to \cite{AgrawalGoyalTS2017}, we show that
    \begin{align*}
            &\Pr\br*{A_t = a, E_{a}^{\mu}\br*{t}, E_{a}^{\theta}\br*{t} \vert \F_{t-1}}\\
            &\qquad\qquad\le q_{a,t}\Pr\br*{A_t = a^*, E_{a}^{\mu}\br*{t}, E_{a}^{\theta}\br*{t} \vert \F_{t-1}}.
    \end{align*}
    Applying Markov's inequality and the above yields
    \begin{align*}
        \sum_{t=1}^{\infty}&\Pr\br*{\Pr\br*{A_t = a, E_{a}^{\mu}\br*{t}, E_{a}^{\theta}\br*{t} \vert \F_{t-1}} > \frac{\epsilon}{3K \Delta_a}}\\
        &\le \frac{3K \Delta_a}{\epsilon}\E\brs*{\sum_{t=1}^{\infty}q_{a,t}\Ind{A_t = a^*}},
    \end{align*}
    and noting that $q_{a,t}$ only updates when $A_t = a^*$, we decouple our analysis from the rounds $t$ by defining 
    {\small
    \begin{align*}
        &\hat{\mu}_k^* \triangleq \frac{1}{k+1}\sum_{s=1}^{k}X_{a^*,s}, \\
        &\tilde{p}_{a,k} \triangleq \Pr\br*{\Ncal\br*{\hat{\mu}_k^*, \frac{1}{1+k}} > y_a \middle\vert X_{a^*,1}, \ldots, X_{a^*,k}},\\
        &\tilde{q}_{a,k} \triangleq \frac{1 - \tilde{p}_{a,k}}{\tilde{p}_{a,k}}.
    \end{align*}}
    We are then left to show that $\sum_{k=0}^{\infty}\E\brs*{\tilde{q}_{a,k}} = \Ocal\br*{\frac{1}{\Delta_a^2}}$.
    To do this, we define a sequence of good events $\brc*{H_{a,k}}_{k}$, showing that $\tilde{q}_{a,k}\Ind{H_{a,k}} \le \Ocal\br*{e^{-\Delta_a^2k}}$, and via the Cauchy-Schwarz inequality that
    \begin{align*}
        \E\brs*{\tilde{q}_{a,k}\Ind{\overline{H_{a,k}}}} \le \sqrt{\E\brs*{\tilde{q}_{a,k}^2}}\sqrt{\Pr\br*{\overline{H_{a,k}}}}.
    \end{align*}
    We prove that $\Pr\br*{\overline{H_{a,k}}}=\Ocal\br*{e^{-\Delta_a^2k}}$ and show that $\E\brs*{\tilde{q}_{a,k}^2}$ is uniformly bounded for all $k$.
    This gives us the desired bound on $S_3$. Finally, combining with the bounds for $S_1$ and $S_2$, and summing over all arms with $\Delta_a \ge \frac{\epsilon}{K}$ yields the desired result.
\end{proofsketch}
An immediate corollary is an instance-independent upper bound for TS that matches the lower bound in \Cref{thm:lower_bound_epsilon} (up to $K$ factors), which is applicable here since knowing $\epsilon$ can only improve our upper bound:
\begin{restatable}{corollary}{TSInstanceIndependent}\label{thm:TS_instance_independent}
    For all $\epsilon > 0$, it holds that:
     \begin{align*}
         \E \brs*{N_{\epsilon} \br*{\pi^{TS \br*{\Ncal}}}} 
        = {\Ocal} \br*{\frac{K^3}{\epsilon^2}}.
     \end{align*}
\end{restatable}
The proof is given in \Cref{appendix:thompson_sampling_is_U-ACE}. We attribute part of the $\Ocal\br*{K^2}$ factor to \Cref{lem:esier_event_to_analyze_TSN}, but note that it was key in our analysis. We leave tightening the $K$-dependence through an alternative approach to future work.
This result shows that TS is instance-independent optimal in $\epsilon$, while \Cref{thm:thompson_sampling_is_U-ACE} matches the lower bound from \Cref{thm:U-ACE_and_consistent_lower_bound} in $\epsilon$, up to logarithmic factors.

\section{Conclusions and Future Work}
This work introduces the expected sample complexity and the ACE framework for its analysis. We showed that U-ACE guarantees imply expected regret bounds and almost-sure convergence to the optimal expected reward, much stronger theoretical guarantees than the popular PAC setting and its various extensions. We also prove that no deterministic policy can be U-ACE uniformly over the $K$-armed Bernoulli class, 
revealing a separation between deterministic and stochastic policies that is absent in other frameworks.

We then characterized the behavior of the expected sample complexity, showing an instance-independent rate in $\epsilon$ of $\tilde{\Theta}\br*{1/\epsilon^2}$ both when $\epsilon$ is known and unknown. Furthermore, we proved that the instance-dependent rates of the two regimes are fundamentally different; notably, in the limit of small $\epsilon$, knowing the value of $\epsilon$ in advance leads to improved expected sample complexity. To our knowledge, this is the first separation result between known-$\epsilon$ and uniform sample complexity bounds.

This work opens up many possible future directions. A natural extension is defining and analyzing the expected sample complexity in other sequential decision-making problems; continuing this line of work in tabular RL can open the door to simplifying expected regret analyses together with convergence results for policy-gradient approaches. Another direction would be defining and analyzing the expected sample complexity in the linear bandits model, an extension of the multi-armed bandit problem, in which stochastic policies have been a topic of many recent works. 

\subsection*{AI use statement}
The setting, \cref{alg:known_epsilon_SE}, and most of the critical results were initially derived by the human authors, including both upper bounds. AI tools were partially used for deriving some of the lower bounds, and some of the original proof arguments were later refined in an iterative interaction with AI tools. 

Additionally, we used generative AI tools for the generation of \Cref{fig:interaction} and the code for \Cref{fig:TS_graph}.
Finally, AI tools were used to verify the correctness of the proofs and revise the paper before submission.
Due to the theoretical nature of our work, the use of AI for generating synthetic data, refining hypotheses, feedback on experiments, translation, dataset cleaning, thematic data analysis, and result interpretation is not applicable.

All LLM-assisted proofs were carefully checked line by line by the authors.
We take responsibility for the final content of this work, including text, claims, or artifacts produced with the aid of generative AI.

\section*{Acknowledgments}
This research was supported by Israel Science Foundation research grant (ISF's No. 4118/25) and the
Maimonides Fund's Future Scientists Center.
\bibliographystyle{plainnat}
\bibliography{references}

\clearpage
\appendix
\thispagestyle{empty}

\onecolumn
\aistatstitle{Supplementary Materials}
\section{Prerequisites}
\subsection{Useful Facts}
\begin{fact}[Hoeffding's Inequality]\label{fact:hoeffding_inequality}
    Let $X_1, \ldots, X_n$ be a sequence of i.i.d. random variables bounded in $\brs*{0, 1}$ almost surely with expectation $\mu$. Denote by $\mu \triangleq \E \brs*{X_1}$, and define $S_n = \sum_{i=1}^{n} X_i$, and $\hat{\mu} = \frac{S_n}{n}$. Then for all $t \ge 0$:
    $$\Pr \br*{S_n - n\mu \ge t} \le \exp \br*{-\frac{2t^2}{n}},$$
    and:
    $$\Pr \br*{\hat{\mu} - \mu \ge t} \le \exp \br*{{-2nt^2}}.$$
\end{fact}
\begin{fact}[Multiplicative Chernoff Bound]\label{fact:Chernoff_bernoulli}
    Let $B_1, \ldots, B_n$ be a sequence of independent Bernoulli random variables. Let $S_n = \sum_{i = 1}^{n} B_i$ and $\mu = \E \brs*{S_n}$. Then for any $\delta \in \br*{0, 1}$:
    $$\Pr \br*{S_n \le \br*{1 - \delta} \mu} \le \exp \br*{-\frac{\delta^2 \mu}{2}}.$$
\end{fact}
\begin{fact}\label{fact:inequality_for_infinite_sums}
    For all $ x> 0$:
    \begin{align*}
        \frac{e^{-x}}{1 - e^{-x}} \le \frac{1}{x},
        \qquad
        \text{and}
        \qquad 
        \frac{1}{1 - e^{-x}} \le 1 + \frac{1}{x}.
    \end{align*}
\end{fact}
\begin{proof}
    Recall that $1 + x \le e^{x}$ for all $x \in \R$. Rearrange to get:
    $$\frac{1}{e^x - 1} \le \frac{1}{x}.$$
    Multiplying both numerator and denominator of the L.H.S. by $e^{-x}$ gives us the first desired inequality; the second inequality is a direct result of the first, as $\frac{1}{1 - e^{-x}}=\frac{e^{-x}}{1 - e^{-x}}+1$.
\end{proof}
\begin{fact}[\cite{abramowitz1964handbook}]\label{fact:Gaussian_normalized_tail_bound}
    Let $Z \sim \Ncal \br*{\mu, \sigma^2}$ be a Gaussian random variable. Then for any $t > 0$:
    \begin{align*}
    \frac{1}{\sqrt{2\pi}}\frac{t}{1 + t^2}e^{-t^2/2}
    \le \Pr \br*{Z - \mu \ge t\sigma} 
    \le \frac{1}{2}e^{-t^2/2}.
    \end{align*}
\end{fact}

\begin{fact}[Lemma 2 of \cite{poland2004convergence}]\label{fact:Ber_KL_bound}
    Let $p, q \in \brs*{1/4, 3/4}$. Then:
    \begin{align*}
        D\br*{Ber(p), Ber(q)} \le \frac{8\br*{p-q}^2}{3}.
    \end{align*}
\end{fact}

\begin{fact}[Lai-Robbins lower bound \citep{Lai_Robbins_1985, BanditAlgsTor2017}]\label{fact:Lai-Robbins_lower_bound}
Let $\Ecal_{Ber}^{K}$ denote the class of Bernoulli bandits with $K$ arms, and let $\pi$ be consistent with respect to this class. 
Then for all $\nu \in \Ecal_{Ber}^{K}$:
\begin{align*}
    \liminf_{T\rightarrow \infty}\frac{R_T(\pi)}{\log(T)} \ge \sum_{a:\Delta_a > 0}\frac{\Delta_a}{d\br*{\mu_a, \mu^*}},
\end{align*}
where $d\br*{b, c} = D\br*{Ber(b), Ber(c)}$. Consequently, by \Cref{fact:Ber_KL_bound}, for instances with means bounded in $\brs*{1/4,3/4}$:
\begin{align*}
    \liminf_{T\rightarrow \infty}\frac{R_T(\pi)}{\log(T)} \ge \sum_{a:\Delta_a > 0}\frac{3}{8\Delta_a},
\end{align*}
\end{fact}

\begin{fact}[Maximal Hoeffding's inequality]\label{fact:Hoeffding_maximal}
    Let $\br*{X_t}_{t \ge 1}$ be a sequence of random variables adapted to a filtration $\F_t$ such that $Y_t \triangleq X_t - \E \brs*{X_t \vert \F_{t-1}} \in \brs*{m_t,M_t}$, with $M_t,m_t$ being $\F_{t-1}$-measurable and $M_t - m_t \le c$ almost surely for all $t$, for some absolute constant $c > 0$.
    Let $\br*{\iota_t}_{t\ge1}$ be $\F_{t-1}$-measurable, and assume that $\iota_t \in \brc*{0,1}$ for all $t\in \N$.
    Define $N_t = \sum_{s\le t} \iota_s$ and $\hat{\mu}_t = \frac{\sum_{s\le t} \iota_s Y_s}{N_t \vee 1}$. Then for any $n, \epsilon > 0$:
    \begin{align*}
        \Pr \br*{\exists t : N_t \ge n \text{ and } \abs{\hat{\mu_t}} \ge \epsilon} \le 2\exp \br*{-\frac{2n\epsilon^2}{c^2}}.
    \end{align*}
\end{fact}
\begin{proof}
    We note that this is a slightly more general version of Proposition 5 in \citep{garivier2022kl}.
    Let $Z_t \triangleq \sum_{s\le t} \iota_s Y_s$, and for $x \in \R$ define $M_t^x \triangleq \exp \br*{xZ_t - \frac{x^2 c^2N_t}{8}}$. We show that $M_t^x$ is a supermartingale. We have that
    {\begin{align}
        \E\brs*{M_t^x\vert \F_{t-1}} 
         &= \E\brs*{\exp \br*{xZ_t - \frac{x^2 c^2N_t}{8}} \middle\vert \F_{t-1}} \nonumber\\
         &= \E\brs*{\exp \br*{x\sum_{s\le t} \iota_s Y_s - \frac{x^2 c^2N_t}{8}} \middle\vert \F_{t-1}} \nonumber\\
         &= \exp \br*{x\sum_{s\le t - 1} \iota_s Y_s - \frac{x^2 c^2N_{t - 1}}{8}}\E\brs*{\exp \br*{x\iota_t Y_t-\frac{x^2 c^2\iota_t}{8}}\middle\vert\F_{t-1}} \nonumber\\
         &= M_{t-1}^x\exp \br*{-\frac{x^2 c^2\iota_t}{8}}\E\brs*{\exp \br*{x\iota_t Y_t} \middle\vert \F_{t-1}}, \label{eq: max-concentration 1}
    \end{align}}
    where the final equality is by the definition of $M_{t-1}^x$ and since $\iota_t$ is $\F_{t-1}$-measurable.
    Note that $\iota_tY_t \in \brs*{\iota_tm_t, \iota_tM_t}$ almost surely, and that $\E\brs*{\iota_tY_t\vert\F_{t-1}} = \iota_t\E\brs*{Y_t\vert\F_{t-1}} = 0$, and so we can apply the conditional version of Hoeffding's lemma on $\iota_tY_t \in \brs*{\iota_tm_t, \iota_tM_t}$ to get
    \begin{align*}
        \E\brs*{\exp\br*{x\iota_tY_t }\middle\vert \F_{t-1}}
        \leq \exp\br*{\frac{x^2\iota_t^2 c^2}{8}}
        =\exp\br*{\frac{x^2\iota_t c^2}{8}},
    \end{align*}
    where the equality holds since $\iota_t \in \brc*{0,1}$, and thus $\iota_t^2 = \iota_t$. 
    Substituting back into \eqref{eq: max-concentration 1}, we get that $ \E\brs*{M_t^x\vert \F_{t-1}} \leq M_{t-1}^x$, and so $M_t^x$ is indeed a supermartingale. As such, it holds that $\E\brs*{M_t^x} \le \E\brs*{M_0^x} = 1$; noting that $M_t^x$ is also non-negative, we can apply Ville's inequality and get that for all $u > 0$
    \begin{align}\label{eq:maximal_heoffding_villes_ineqaulity}
        \Pr \br*{\exists t : M_t^x \ge u} \le \frac{\E\brs*{M_0^x}}{u} \le \frac{1}{u}.
    \end{align}
    Therefore, for all $\epsilon > 0$ and $n \in \N$, taking $x = \frac{4\epsilon}{c^2}$ yields
    \begin{align*}
        \Pr\br*{\exists t : Z_t \ge N_t\epsilon, N_t \ge n}
        &= \Pr\br*{\exists t : \exp \br*{xZ_t - \frac{x^2 c^2N_t}{8}} \ge \exp \br*{xN_t \epsilon - \frac{x^2 c^2N_t}{8}}, N_t \ge n}\\
        &= \Pr\br*{\exists t : M_t^x \ge \exp \br*{xN_t \epsilon - \frac{x^2 c^2N_t}{8}}, N_t \ge n}\\
        &= \Pr\br*{\exists t : M_t^x \ge \exp \br*{\frac{2}{c^2}N_t\epsilon^2}, N_t \ge n} \tag{definition of $x$}\\
        &\le \Pr\br*{\exists t : M_t^x \ge \exp \br*{\frac{2}{c^2}n\epsilon^2}}\\
        &\le \exp\br*{-\frac{2n\epsilon^2}{c^2}}. \tag{\Cref{eq:maximal_heoffding_villes_ineqaulity}}
    \end{align*}
    We can apply similar arguments to get $\Pr\br*{\exists t : Z_t \le -N_t\epsilon, N_t \ge n} \le \exp\br*{-\frac{2n\epsilon^2}{c^2}}$, and finally noting that $\hat{\mu}_t = Z_t/N_t$ gives us that for all $n, \epsilon > 0$
    \begin{align*}
        \Pr \br*{\exists t : N_t \ge n \text{ and } \abs{\hat{\mu_t}} \ge \epsilon} 
        \le \Pr\br*{\exists t : Z_t \le -N_t\epsilon, N_t \ge n} + \Pr\br*{\exists t : Z_t \ge N_t\epsilon, N_t \ge n} 
        \le 2\exp\br*{-\frac{2n\epsilon^2}{c^2}},
    \end{align*}
    as desired.
\end{proof}

\clearpage

\subsection{Interaction Protocol and Filtrations}\label{sec:interaction-filtration}
One of the distinguishing features of this work is the analysis of the policy itself rather than the realized action. This raises the need to carefully define how histories and internal randomization of the algorithm work.

We explicitly represent the internal randomization of the algorithm by a sequence of mutually independent random variables $(U_t)_{t\ge1}$, 
independent of the reward table, such that $U_t$ contains all randomization used by the algorithm in round $t$ before observing the reward. 
We let
\begin{align*}
    \F_t \triangleq \sigma\br*{A_1,X_1, \ldots,A_t, X_t},
\end{align*}
denote the history up to the end of round $t$, with $\F_0$ taken to be the trivial sigma-algebra.
At the beginning of round $t$, the policy $\pi_t$ is $\F_{t-1}$-measurable, and conditioned on the history,
\begin{align*}
    A_t \vert \F_{t-1} \sim \pi_t\br*{\cdot\vert \F_{t-1}}.
\end{align*}
Noting that the action $A_t$ is completely determined by $\F_{t-1}$ and $U_{t}$, and since the action space is finite, we may write
\begin{align*}
    A_t = \Phi_t\br*{\pi_t, U_t}
\end{align*}
for some measurable map $\Phi_t$. For $t \ge 0$, we define the filtration
\begin{align*}
    \Gcal_t \triangleq \sigma\br*{\F_{t}, U_1, \ldots, U_{t + 1}}
\end{align*}
to be the history available at round $t$ together with the realized internal randomness.
Hence, immediately before observing the reward in round $t$, the available
information is $\Gcal_{t-1}$. In particular, all random choices determining the action in round $t$ are
in $\Gcal_{t-1}$, and in particular $A_t$ is $\Gcal_{t-1}$-measurable, giving us that $\E\brs*{X_t\vert \Gcal_{t-1}}=\mu_{A_t}.$

\begin{figure}[h]
    \centering
    \includegraphics[width=1\textwidth]{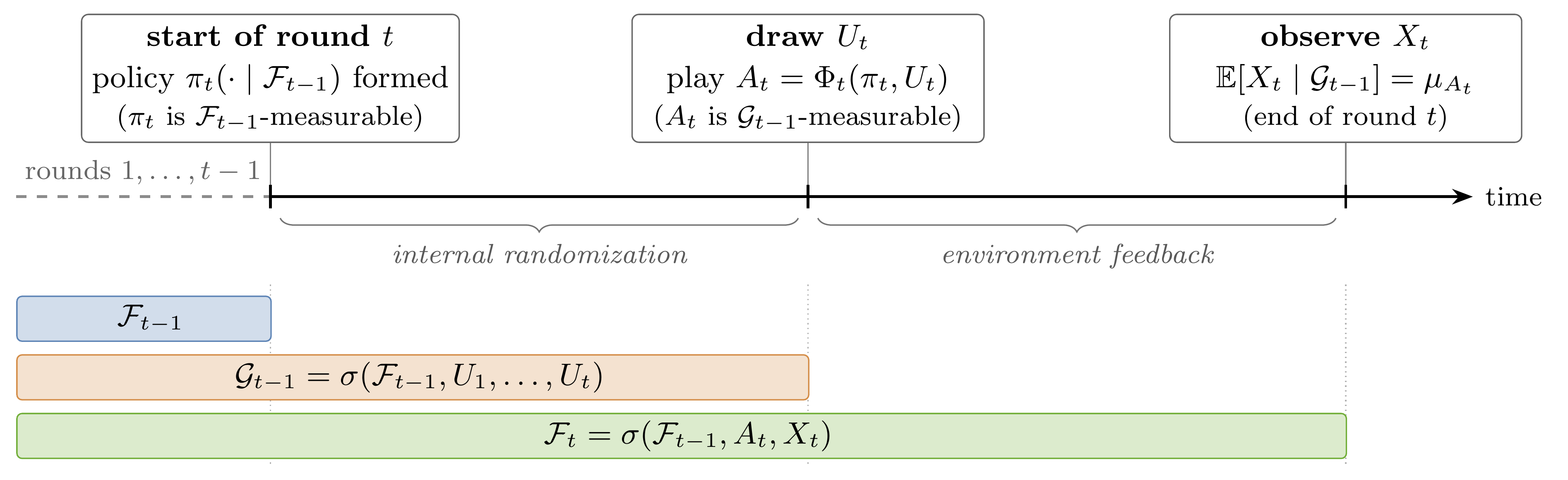}
    \caption{
    Flow of randomness in round $t$ of the interaction.
    The policy $\pi_t$ is determined from the past history $\F_{t-1}$. The internal randomization $U_t$ is then
    revealed, making the realized action $A_t$ and any other internal choice mechanism
    predictable before the reward $X_t$ is observed.
    }
    \label{fig:interaction}
\end{figure}

\subsection{The Table Model}
The sequential nature of the multi-armed bandit problem requires a delicate approach to the analysis of algorithms. 
Towards this, in the analysis of all our algorithms, we adopt an infinite table model to generate our samples; details on why a policy is well-defined for such a construction can be found in \citep[Section 4.6]{BanditAlgsTor2017}.
Before the interaction, we generate $\br{\tilde{X}_{a,k}}_{a \in [K], k \in \N}$ such that the law of $\tilde{X}_{a,k}$ is $\nu_a$ (independent of all other samples). During the interaction, we let $X_t = \tilde{X}_{A_t, n_t(A_t) + 1}$, noting that $A_t, n_t(A_t)$ are $\Gcal_{t-1}$-measurable, thus making this construction valid. We further denote by $\tilde{\mu}_{a,k} \triangleq \frac{1}{k}\sum_{s=1}^{k}\tilde{X}_{a,s}$ the empirical mean of the first $k$ samples from the row in the table representing arm $a$. 

\clearpage
\section{Properties of U-ACE -- Omitted Proofs}
\subsection{Proof of claim \ref{thm:expected_regret_equivalent_for_policy_gap}}\label{appendix:expected_regret_equivalent_for_policy_gap}
\expectedRegretEquivalentForPolicyGap*
\begin{proof}
First, observe that the expected regret can be rewritten as
\begin{align*}
    R_{T}\br*{\pi,\nu} 
    = \E\brs*{\sum_{t=1}^{T}{\Delta_t}}
    = \E\brs*{\sum_{t=1}^{T}\sum_{a \in \brs*{K}}{\Delta_t}\Ind{A_t = a}}
    = \sum_{a \in \brs*{K}}{\Delta_a}\E\brs*{\sum_{t=1}^{T}\Ind{A_t = a}}
    = \sum_{a \in \brs*{K}}\Delta_a \E \brs*{n_{T+1}(a)}.
\end{align*}
Furthermore, observe that for all $a\in\brs*{K}$:
\begin{align*}
\E \brs*{n_{T+1} \br*{a}} 
&\overset{(1)}{=} \E \brs*{\sum_{t=1}^{T} \Ind{A_t = a}} 
= \sum_{t=1}^{T} \E \brs*{ \Ind{A_t = a}}
= \sum_{t=1}^{T} \E \brs*{\Pr \br*{{A_t = a} | \F_{t-1}}} 
\overset{(2)}{=} \sum_{t=1}^{T} \E \brs*{\pi_t \br*{a \mid \F_{t-1}}},
\end{align*}
where $(1)$ is by definition and $(2)$ is also by definition. 
Combining the above yields
\begin{align*}
\sum_{a \in \brs*{K}} \Delta_a \E \brs*{n_{T+1}(a)}
= \sum_{a \in \brs*{K}} \Delta_a \sum_{t=1}^{T} \E \brs*{\pi_t \br*{a \mid \F_{t-1}}} 
= \E \brs*{ \sum_{t=1}^{T} \sum_{a \in \brs*{K}} \Delta_a  \pi_t \br*{a \mid \F_{t-1}} } 
= \E \brs*{ \sum_{t=1}^{T} \Delta_{\pi_t}},
\end{align*}
where the last equality is by definition.
\end{proof}

\clearpage

\subsection{Proof of Proposition \ref{thm:ACE_implies_AS_boundedness}}\label{appendix:ACE_implies_AS_boundedness}
\ACEImpliesASBoundedness*
\begin{proof}
First, note that for any $\epsilon > 0$, by Fubini-Tonelli:
\begin{align*}
    \E \brs*{\sum_{t=1}^{\infty}\Ind{\Delta_{\pi_t} > \epsilon}} 
    = \sum_{t=1}^{\infty}\E \brs*{\Ind{\Delta_{\pi_t} > \epsilon}} 
    = \sum_{t=1}^{\infty}\Pr\br*{\Delta_{\pi_t} > \epsilon}.
\end{align*}
By $\epsilon$-ACE and the above, we have that
\begin{align*}
    \sum_{t=1}^{\infty}\Pr\br*{\Delta_{\pi_t} > \epsilon} < \infty,
\end{align*}
and so by the first lemma of Borel-Cantelli:
\begin{align*}
    \Pr (\underbrace{\forall \tau, \exists t > \tau \enspace{ s.t. }\enspace {\Delta_{\pi_t} > \epsilon}}_{\triangleq A}) = 0,
\end{align*}
which in turn implies that
\begin{align*}
    1 
    = \Pr\br*{\overline{A}}
    =\Pr \br*{\exists \tau\enspace{s.t.}\enspace \forall t > \tau: \brc*{\Delta_{\pi_t} \le \epsilon}} 
    \le \Pr \br*{\limsup_{t \rightarrow \infty}\Delta_{\pi_t} \le \epsilon}
    = 1,
\end{align*}
giving us our result.
\end{proof}
\UACEImpliesASConvergence*
\begin{proof}
    We apply the fact that $\pi$ is U-ACE for the sequence $\epsilon_k = \frac{1}{k}$, and combine with \Cref{thm:ACE_implies_AS_boundedness} to get
    \begin{align*}
       \forall k \in \N : \Pr \br*{\limsup_{t \rightarrow \infty} \Delta_{\pi_t} \le \frac{1}{k}} = 1.
    \end{align*}
    Recalling that a countable intersection of almost-sure events is an almost-sure event, the above gives us
    \begin{align*}
        \Pr \br*{\bigcap_{k=1}^{\infty}\brc*{\limsup_{t \rightarrow \infty} \Delta_{\pi_t} \le \frac{1}{k}}} = 1.
    \end{align*}
    Noting that $\Delta_{\pi_t} \ge 0$ by definition, and that
    \begin{align*}
        \bigcap_{k=1}^{\infty}\brc*{\limsup_{t \rightarrow \infty} \Delta_{\pi_t} \le \frac{1}{k}}
        \subseteq \brc*{\limsup_{t \rightarrow \infty} \Delta_{\pi_t} \le 0},
    \end{align*} 
    yields $\Pr \br*{\lim_{t \rightarrow \infty} \Delta_{\pi_t} = 0} = 1$, as desired.
\end{proof}

\clearpage

\subsection{Proof of Claim \ref{thm:UACE_implies_expected_regret}}\label{appendix:UACE_implies_expected_regret}
\UACEImpliesExpectedReget*
\begin{proof}
Following a similar decomposition to Claim 1 in \cite{LenientRegretMerlis2021}:
{\small
\begin{align*}
    R_T(\pi) 
    = \E\brs*{\sum_{t=1}^{T}\Delta_{\pi_t}} 
    = \E\brs*{\sum_{t=1}^{T}\int_{\epsilon=0}^{1}\Ind{\Delta_{\pi_t} > \epsilon}d\epsilon}
    = \int_{\epsilon=0}^{1}\E\brs*{\sum_{t=1}^{T}\Ind{\Delta_{\pi_t} > \epsilon}}d\epsilon
    \le \int_{\epsilon=0}^{1}\min\brc*{T, \E\brs*{\sum_{t=1}^{\infty}\Ind{\Delta_{\pi_t} > \epsilon}}}d\epsilon,
\end{align*}}
which by definition implies that
\begin{align*}
    R_T(\pi) 
    \le \int_{\epsilon=0}^{1}\min\brc*{T, \E\brs*{N_{\epsilon}(\pi)}}d\epsilon.
\end{align*}
We begin with the case that $\alpha > 1$. Note that for $T \leq C_{\nu}$, the bound holds trivially. For $T > C_{\nu}$, we have that $\br*{C_{\nu}/T}^{1/\alpha} < 1$, and we are given that $\E\brs*{\sum_{t=1}^{\infty}\Ind{\Delta_{\pi_t} > \epsilon}} \le C_{\nu}\epsilon^{-\alpha}$ for all $\epsilon > 0$, which yields
\begin{align*}
    \int_{\epsilon=0}^{1}\min\brc*{T, \E\brs*{N_{\epsilon}(\pi)}}d\epsilon
    \le \int_{\epsilon=0}^{1}\min\brc*{T, C_{\nu}\epsilon^{-\alpha}}d\epsilon
    \le \int_{\epsilon=0}^{\br*{C_{\nu}/T}^{1/\alpha}} T d\epsilon + \int_{\epsilon=\br*{C_{\nu}/T}^{1/\alpha}}^{1} C_{\nu}\epsilon^{-\alpha} d\epsilon,
\end{align*}
and a simple computation gives us
\begin{align*}
    \int_{\epsilon=0}^{\br*{C_{\nu}/T}^{1/\alpha}} T d\epsilon + \int_{\epsilon=\br*{C_{\nu}/T}^{1/\alpha}}^{1} C_{\nu}\epsilon^{-\alpha} d\epsilon
    = C_{\nu}^{1/\alpha}T^{\frac{\alpha - 1}{\alpha}} + \frac{C_{\nu}}{1 - \alpha}\br*{1 - \br*{C_{\nu}/T}^{\frac{1 - \alpha}{\alpha}}}
    \le \frac{\alpha}{\alpha - 1} C_{\nu}^{1/\alpha}T^{\frac{\alpha - 1}{\alpha}}.
\end{align*}
For $\alpha = 1$, we note that if $C_\nu \geq T$, then $R_T(\pi) \le T \leq C_\nu$ and so the bound holds. For $C_\nu < T$, we have that:
\begin{align*}
    \int_{\epsilon=0}^{1}\min\brc*{T, \E\brs*{N_{\epsilon}(\pi)}}d\epsilon
    \le \int_{\epsilon=0}^{1}\min\brc*{T, C_{\nu}\epsilon^{-1}}d\epsilon
    \le \int_{\epsilon=0}^{\br*{C_{\nu}/T}} T d\epsilon + \int_{\epsilon=\br*{C_{\nu}/T}}^{1} C_{\nu}\epsilon^{-1} d\epsilon,
\end{align*}
and thus
\begin{align*}
    \int_{\epsilon=0}^{\br*{C_{\nu}/T}} T d\epsilon + \int_{\epsilon=\br*{C_{\nu}/T}}^{1} C_{\nu}\epsilon^{-1} d\epsilon
    \le C_{\nu} + C_{\nu}\log\br*{\frac{T}{C_{\nu}}}. 
\end{align*}
\end{proof}

\clearpage

\subsection{Proof of Claim \ref{thm:deterministic_alg_sublinear_regret_and_U-ACE}}\label{appendix:deterministic_alg_sublinear_regret_and_U-ACE}
\deterministicAlgSublinearRegretAndUACE*
\begin{proof}
    We know the policy is a deterministic mapping from $\F_{t-1}$ to $\brs*{K}$, and as such for any history there exists $a \in \brs*{K}$ such that $\pi_t (a \vert \F_{t-1}) = 1$. Further, for any instance $\nu$ and for all $\epsilon < \Delta_{\min}$:
    \begin{align*}
        R_T\br*{\pi,\nu} 
        \overset{(1)}{=} \E\brs*{\sum_{t=1}^{T}\Delta_{A_t}}
        \overset{(2)}{\le} \E\brs*{\sum_{t=1}^{T}\Ind{\Delta_{A_t}> \epsilon}}
        \overset{(3)}{\le} \E\brs*{\sum_{t=1}^{\infty}{\Ind{\Delta_{A_t} > \epsilon}}}
        \overset{(4)}{=} \E\brs*{\sum_{t=1}^{\infty}{\Ind{\Delta_{\pi_t} > \epsilon}}}
        = \E\brs*{N_{\epsilon}\br*{\pi,\nu}},
    \end{align*}
    where (1) is by definition, (2) is since each time an arm with $\Delta_{A_t} > \epsilon$ is chosen, it incurs at most $1$ regret (since the distributions are Bernoulli), (3) is by the monotone convergence theorem and the non-negativity of the indicators, and (4) is since the policy is deterministic.
    Using this lower bound, we divide into two cases:
    \paragraph{If $\pi$ is consistent.} Fix $\nu \in \Ecal_{Ber}^{K}$ with $\Delta_{\min} > 0$ and $\mu^*<1$. From the lower bound of Lai-Robbins (\Cref{fact:Lai-Robbins_lower_bound}) we have that:
    \begin{align*}
        \liminf_{T \rightarrow \infty}\frac{R_T\br*{\pi,\nu}}{\ln(T)} > 0,
    \end{align*}
    which in turn implies that $\lim_{T\rightarrow \infty} R_T\br*{\pi,\nu} = \infty$, and thus by the lower bound $\E\brs*{N_{\epsilon}\br*{\pi,\nu}} = \infty$, proving that $\pi$ cannot be $\epsilon$-ACE in this case.
    \paragraph{If $\pi$ is not consistent.} Then there exists $p, c > 0$, instance $\nu \in \Ecal_{Ber}^{K}$ 
    and a strictly increasing sequence of natural numbers $\brc*{T_i}_{i \in \N}$ such that $R_{T_i}(\pi,\nu) \ge cT_i^{p}$. 
    Together with the monotonicity of the expected regret, this gives $\lim_{T \rightarrow \infty}R_T(\pi,\nu) = \infty$.
    Thus for all $\epsilon < \Delta_{\min}$, by the above we will again get $\E\brs*{N_{\epsilon}\br*{\pi,\nu}} = \infty$, 
    proving our desired result.
\end{proof}

\newpage
\section{Non-Uniform Algorithms -- Omitted Proofs}
\subsection{Proof of Theorem \ref{thm:known_epsilon_upper_bound}}\label{appendix:known_epsilon_upper_bound}
\KnownEpsilonUpperBound*
\begin{proof} For brevity, we write $\pi = \pi^{\text{SET$\epsilon$-G}}$.
Note that when $\epsilon \ge 1$, since $\Delta_{\pi_t} \le 1$, we have that $N_{\epsilon}(\pi) = 0$ a.s., thus $\E\brs*{N_{\epsilon}(\pi)} = 0$, which trivially satisfies the upper bound.
We can therefore assume that $\epsilon \in (0,1)$ and is known to the algorithm. We denote 
\begin{align*}
    \Acal^* \triangleq \brc*{a \in \brs*{K} : \Delta_a = 0},
    \qquad
    \Acal_{sub} \triangleq \brc*{a \in \brs*{K} : \Delta_a > 0},
\end{align*}
noting that if $\mathcal A_{\mathrm{sub}} = \emptyset$, then again $N_{\epsilon}(\pi) = 0$ a.s., thus $\E\brs*{N_{\epsilon}(\pi)} = 0$,
which trivially satisfies the upper bound.
We henceforth assume that $\mathcal A_{\mathrm{sub}} \ne \emptyset$, let $q \triangleq \frac{\epsilon^2}{16K}$ and define for all $n \ge 1$ the confidence radius
    \begin{align*}
        \beta_n 
        \triangleq\sqrt{\frac{1}{2n}\ln\br*{\frac{4Kn^2}{q}}}
        = \sqrt{\frac{1}{2n}\ln\br*{\frac{64K^2n^2}{\epsilon^2}}}.
    \end{align*}
    \begin{restatable}{lemma}{BetaDecreasingInn}\label{lem:Beta_decreasing}
        $\beta_n$ is decreasing in $n$.
    \end{restatable}
    Proof of this lemma can be found in \Cref{subsec:known_epsilon_lemmas}.
For $x \in (0,1]$, define $n(x)\triangleq \min\brc*{n \ge 1 :\beta_n < \frac{x}{8}}$, and let $n_\epsilon\triangleq n\br*{\frac{\epsilon}{4}}$ be the maximal number of successive-elimination epochs.
For our analysis, we define the good event to be
\begin{align*}
    G \triangleq \bigcap_{a \in [K]}\bigcap_{n=1}^{\infty}\brc*{\abs{\hat{\mu}_{a,n}-\mu_a}\le\beta_n}.
\end{align*}
Note that by Hoeffding's inequality (\Cref{fact:hoeffding_inequality}) and the definition of $\beta_n$:
\begin{align*}
    \Pr\br*{\abs{\hat{\mu}_{a,n}-\mu_a}>\beta_n}
    \le2e^{-2n\beta_n^2} 
    =2\exp\br*{-\ln\br*{\frac{4Kn^2}{q}}}
    =\frac{q}{2Kn^2},
\end{align*}
and so by the union bound,
\begin{align*}
    \Pr\br*{\overline{G}}
    \le\sum_{a=1}^K\sum_{n=1}^{\infty}\frac{q}{2Kn^2} 
    =\frac{q}{2}\sum_{n=1}^{\infty}\frac{1}{n^2}
    \le q.
    \end{align*}
We decompose the expected sample complexity of $\pi$ into the two phases: the successive elimination (SE) phase and the $\epsilon$-Greedy ($\epsilon$-G) phase:
\begin{align*}
    \E\brs*{N_{\epsilon}(\pi)} 
    = \E\brs*{\sum_{t=1}^{\infty}\Ind{\Delta_{\pi_t}>\epsilon}} 
    = \underbrace{\E\brs*{\sum_{t=1}^{T_{SE}}\Ind{\Delta_{\pi_t}>\epsilon}}}_{{SE}} 
    + \underbrace{\E\brs*{\sum_{t=T_{SE} + 1}^{\infty}\Ind{\Delta_{\pi_t}>\epsilon}}}_{{\epsilon-G}},
\end{align*}
where $T_{SE}$ denotes the round in which the SE phase terminated.
We analyze each phase of the algorithm separately, starting with the SE phase.

Our analysis follows a simple argument, showing that under the good event, bad arms are eliminated proportionally to their suboptimality, while the probability of the bad event balances the worst-case expected sample complexity in this phase.

\paragraph{Optimal arms are not eliminated under $G$}
Let $a^* \in \Acal^*$ be some optimal arm. We show that under the event $G$, $a^*$ is never eliminated in the SE phase.
Suppose $a^*$ is active at epoch $n$. For any other active arm $b$, it holds that:
\begin{align*}
    \hat{\mu}_{b,n} \le \mu_b + \beta_n \le \mu^* + \beta_n.
\end{align*}
Moreover, note that $\hat{\mu}_{a^*, n} \ge \mu^* - \beta_n$, and thus $\hat{\mu}_{b,n} - \hat{\mu}_{a^*,n} \le 2\beta_n < 6\beta_n$, 
implying that $a^*$ is not eliminated at the end of epoch $n$. This is true for any $n$, and since all arms are active in the first epoch, we have that $a^*$ is never eliminated in the SE phase under the good event $G$.

\paragraph{Suboptimal arms are eliminated proportionally to their suboptimality under $G$}
Now, fix a suboptimal arm $a \in \Acal_{sub}$. We show that under $G$, $a$ is eliminated after at most $n(\Delta_a)$ epochs.
Suppose $a$ is active at epoch $n = n(\Delta_a)$. Thus, on $G$:
\begin{align*}
    \hat{\mu}_{a^*, n} - \hat{\mu}_{a, n} 
    \ge \mu^* - \beta_n - \br*{\mu_a + \beta_n} 
    = \Delta_a - 2\beta_n.
\end{align*}
By the definition of $n$, we have that $\beta_n < \frac{\Delta_a}{8}$, thus we get
\begin{align*}
    \hat{\mu}_{a^*, n} - \hat{\mu}_{a, n} \ge \Delta_a - \frac{\Delta_a}{4} = \frac{3\Delta_a}{4}.
\end{align*}
Since $6\beta_{n} \le \frac{3\Delta_a}{4}$, $a$ is therefore eliminated by $a^*$ after at most ${n(\Delta_a)}$ epochs if $n_{\epsilon} > n(\Delta_a)$. 

This implies that on $G$, every suboptimal arm $a$ is pulled for at most $\min\brc*{n(\Delta_a), n_{\epsilon}}$ epochs during the SE phase. Moreover, since the policy in this phase is deterministic, we have that:
\begin{align*}
    \E\brs*{\sum_{t=1}^{T_{SE}}\Ind{\Delta_{\pi_t}>\epsilon}\Ind{G}} 
    \le \sum_{a : \Delta_a > \epsilon} \min\brc*{n(\Delta_a), n_{\epsilon}} 
    \le \sum_{a : \Delta_a > \epsilon} n(\Delta_a),
\end{align*}
noting that if all arms have $\Delta_{a} \le \epsilon$, then $\Delta_{\pi_t} \le \epsilon$ and as such $N_{\epsilon}(\pi) = 0$, aligning with the above.
Furthermore, we can trivially bound the sample complexity by noting that each epoch contains at most $K$ rounds, and the total number of epochs is bounded by $n_{\epsilon}$, implying that $\sum_{t=1}^{T_{SE}}\Ind{\Delta_{\pi_t}>\epsilon} \le K n_{\epsilon}$ almost surely, which gives us that
\begin{align*}
    \E\brs*{\sum_{t=1}^{T_{SE}}\Ind{\Delta_{\pi_t}>\epsilon}\Ind{\overline{G}}} 
    \le \Pr\br*{\overline{G}}K n_{\epsilon} 
    \le q K n_{\epsilon}.
\end{align*}
We give a bound for $n(x)$ in the following lemma:
\begin{restatable}{lemma}{SEHorizonBound}\label{lem:known_epsilon_SE_horizon_bound}
    For every $x\in(0,1]$,
\begin{align*}
    n(x)\le \ceil*{\frac{2048}{x^2}\ln\br*{\frac{1024K}{\epsilon x}}},
    \qquad
    and
    \qquad
    n_\epsilon \le \ceil*{\frac{32768}{\epsilon^2}\ln\br*{\frac{4096K}{\epsilon^2}}}.
\end{align*}
\end{restatable}
Proof of \Cref{lem:known_epsilon_SE_horizon_bound} can be found in \Cref{subsec:known_epsilon_lemmas}.
Applying this result to the previous bound, we get
\begin{align*}
    \E\brs*{\sum_{t=1}^{T_{SE}}\Ind{\Delta_{\pi_t}>\epsilon}}
    &= \E\brs*{\sum_{t=1}^{T_{SE}}\Ind{\Delta_{\pi_t}>\epsilon}\Ind{G}}
    + \E\brs*{\sum_{t=1}^{T_{SE}}\Ind{\Delta_{\pi_t}>\epsilon}\Ind{\overline{G}}} \\
    &\le \sum_{a : \Delta_a > \epsilon} n(\Delta_a) + q K n_{\epsilon} \\
    &\le \sum_{a : \Delta_a > \epsilon} \ceil*{\frac{2048}{\Delta_a^2}\ln\br*{\frac{1024K}{\epsilon \Delta_a}}} 
    + \frac{\epsilon^2}{16} \ceil*{\frac{32768}{\epsilon^2}\ln\br*{\frac{4096K}{\epsilon^2}}},
\end{align*}
implying that
\begin{align}\label{eq:known_epsilon_SE_bound}
    \E\brs*{\sum_{t=1}^{T_{SE}}\Ind{\Delta_{\pi_t}>\epsilon}} 
    = \Ocal\br*{\sum_{a : \Delta_a > \epsilon}{\frac{1}{\Delta_a^2}\ln\br*{\frac{K}{\epsilon \Delta_a}}} + {\ln\br*{\frac{K}{\epsilon^2}}}}.
\end{align}
Noting that for $\Delta_a > \epsilon$ we have $\ln\br*{\frac{K}{\epsilon\Delta_a}} \le \ln\br*{\frac{K}{\epsilon^2}}$, we can absorb the additional term since $\frac{1}{\Delta_a^2} \ge 1$ to get the desired bound of $\Ocal\br*{\sum_{a : \Delta_a > \epsilon}{\frac{1}{\Delta_a^2}\ln\br*{\frac{K}{\epsilon^2}}}}$.

Next, we analyze the $\epsilon$-greedy phase of the algorithm.

Our analysis is again relatively simple; we show that the probability of a bad arm being empirically best decays exponentially in the number of rounds, but is also bounded by $q$,
which creates a balance between the factors in $\epsilon$ and in $\Delta_a$, which leads us to our desired result.

Let $s = 1, 2,\ldots$ denote the time in the $\epsilon$-greedy phase, and $g_s$ denote the greedy arm at round $s$.
By the definition of $\pi_s$, we have that
\begin{align*}
    \Delta_{\pi_s} = \br*{1 - \frac{\epsilon}{4}}\Delta_{g_s} + \frac{\epsilon}{4K}\sum_{a\in\brs*{K}}\Delta_a \le \Delta_{g_s} + \frac{\epsilon}{4},
\end{align*}
which implies that $\brc*{\Delta_{\pi_s} > \epsilon} \subseteq \brc*{\Delta_{g_s} > \frac{3\epsilon}{4}}$. 
We show the following:
\begin{lemma}\label{lem:known_epsilon_uniform_bound_prob}
    For all $s$ and for any arm with $\Delta_a > \frac{3\epsilon}{4}$, it holds that
    \begin{align*}
        \Pr\br*{g_s =a} \le q.
    \end{align*}
\end{lemma}
\begin{proof}
We begin by showing that under $G$, if an arm survives the SE phase, then it must hold that $\Delta_a \le \epsilon/4$. 

Let $a$ be an arm that is not eliminated in the SE phase. We divide into cases:

If the SE phase did not terminate early, then by definition we terminate at epoch $n_{\epsilon}$, in which we have $\beta_{n_{\epsilon}} \le \frac{\epsilon}{32}$.
Suppose towards contradiction that $\Delta_a > \epsilon/4 \ge 8\beta_{n_{\epsilon}}$. For arm $a$ that survived to epoch $n_{\epsilon}$, we have that on $G$,
\begin{align*}
    \hat{\mu}_{a^*,n_{\epsilon}} - \hat{\mu}_{a,n_{\epsilon}}
    \ge \mu^* - \beta_n - \br*{\mu_a + \beta_n} 
    = \Delta_a - 2\beta_n
    > 6\beta_{n_{\epsilon}},
\end{align*}
which implies that $a$ is eliminated, which is a contradiction. Thus $\Delta_a \le \frac{\epsilon}{4}$.

If the SE phase did terminate early, then by definition there is only one active arm remaining, which must be the optimal arm, since an optimal arm is never eliminated on $G$. 

In either case, for all surviving arms, it holds that $\Delta_a \le \epsilon/4$.

Next, we show that eliminated arms never become empirical best on $G$. 
Let $a$ be a suboptimal arm which was eliminated at some epoch $n$. Then there exists an arm $b$ such that $\hat{\mu}_{b,n} - \hat{\mu}_{a,n} > 6\beta_n$.

On $G$, we have that $\hat{\mu}_{b,n} \le \mu_b + \beta_n \le \mu^* + \beta_n$, while $\hat{\mu}_{a,n} \ge \mu_a -\beta_n = \mu^* - \Delta_a - \beta_n$.
Combining, we get
\begin{align*}
     \mu^* + \beta_n - (\mu^* - \Delta_a - \beta_n) > 6\beta_n.
\end{align*}
This implies that $\Delta_a/4 > \beta_n$. Note that $a^*$ is still active at epoch $n$, thus has at least $n$ samples. Furthermore, arm $a$ also has at least $n$ samples. At every future time $t$, their sample count can only increase, and so by \Cref{lem:Beta_decreasing} we have that
\begin{align*}
    \beta_{n_t(a)} \le \beta_n, \qquad \beta_{n_t(a^*)} \le \beta_n.
\end{align*}
Thus, on $G$:
\begin{align*}
    \hat{\mu}_t(a^*) - \hat{\mu}_t(a) 
    \ge \br*{\mu^* -  \beta_{n_t(a^*)}} - \br*{\mu_a + \beta_{n_t(a)}}
    \ge \br*{\mu^* - \beta_n} - \br*{\mu_a + \beta_n}
    = \Delta_a - 2\beta_n
    > \Delta_a/2 > 0,
\end{align*}
Therefore, eliminated arms will not become empirically best after the SE phase on $G$.

We have shown that arms that survive the entire SE phase have $\Delta_a \le \epsilon/4$, and that arms that are eliminated do not become empirically best again, which implies that under $G$, in round $s$ of the $\epsilon$-greedy phase, the chosen empirical-best arm $g_s$ satisfies $\Delta_{g_s} \le \epsilon/4$. Taking the complement gives us
\begin{align*}
    \brc*{\Delta_{g_s} > \frac{3\epsilon}{4}} \subseteq \brc*{\Delta_{g_s} > \frac{\epsilon}{4}} \subseteq \overline{G},
\end{align*}
yielding the desired result.
\end{proof}
\Cref{lem:known_epsilon_uniform_bound_prob} gives a uniform bound on the probability of choosing a bad arm. The next lemma gives us a decaying bound that allows us to sum the tail of the expected sample complexity:
\begin{lemma}\label{lem:known_epsilon_decaying_bound_prob}
    For all $s$, and for any arm $a$ with $\Delta_a > \frac{3\epsilon}{4}$, it holds that
    \begin{align*}
    \Pr\br*{g_s = a} \le 6e^{-\frac{\epsilon(s-1)\Delta_a^2}{32K}}.
\end{align*}
\end{lemma}
\begin{proof}
 Fix a suboptimal arm $a$ satisfying $\Delta_a > \frac{3\epsilon}{4}$, and let $a^* \in \Acal^*$.
 Denote by $M_s(b)$ the number of exploration rounds in which arm $b$ was chosen during the first $s-1$ $\epsilon$-greedy phase rounds, and note that by the definition of the policy 
\begin{align*}
    M_s(a), M_s(a^*) \sim Bin\br*{s-1, \frac{\epsilon}{4K}}.
\end{align*}
Let $m_s \triangleq \frac{(s-1)\epsilon}{4K}$. Lastly, let $N_s(b)$ denote the total number of draws of arm $b$ before round $s$ of the $\epsilon$-greedy phase, including all pulls from the SE phase.  As such, the empirical mean used in the algorithm for arm $b$ is given by
\begin{align*}
    \hat{\mu}^{\pi}_s(b) = \frac{1}{N_s(b)}\sum_{j=1}^{N_s(b)}\tilde{X}_{b,j}.
\end{align*}
We denote:
    \begin{align*}
        B_{b,s} \triangleq \brc*{\abs{\hat{\mu}^{\pi}_s(b) - \mu_b} \ge \frac{\Delta_a}{2}}, 
        \qquad D_{b,s} = \brc*{M_s(b) \ge \frac{m_s}{2}}.
    \end{align*}
Note that on the event $B_{a,s}^C \cap B_{a^*,s}^C$, we have that
\begin{align*}
    \hat{\mu}^{\pi}_s(a) < \mu_a + \frac{\Delta_a}{2} = \mu^* - \frac{\Delta_a}{2} < \hat{\mu}^{\pi}_s(a^*),
\end{align*}
implying that arm $a$ is not empirically best at round $s$, and specifically $g_s \ne a$. Taking the complement and including the events $D_{a,s}, D_{a^*,s}$, we get:
\begin{align*}
    \brc*{g_s = a} \subseteq B_{a,s} \cup B_{a^*,s} \subseteq \br*{D_{a,s}^C \cup D_{a^*,s}^C} \cup \br*{B_{a,s} \cap D_{a,s}} \cup \br*{B_{a^*,s} \cap D_{a^*,s}}, 
\end{align*}
and thus by the union bound:
\begin{align*}
    \Pr\br*{g_s =a} \le \underbrace{\Pr\br*{D_{a,s}^C \cup D_{a^*,s}^C}}_{(*)} + \underbrace{\Pr\br*{B_{a,s} \cap D_{a,s}} + \Pr\br*{B_{a^*,s} \cap D_{a^*,s}}}_{(**)}.
\end{align*}
To bound $(*)$, we apply the union bound and the multiplicative Chernoff bound (\Cref{fact:Chernoff_bernoulli}) and get that
\begin{align*}
    \Pr\br*{D_{a,s}^C \cup D_{a^*,s}^C} \le \Pr\br*{M_s(a) < \frac{m_s}{2}} + \Pr\br*{M_s(a^*) < \frac{m_s}{2}} \le 2e^{-m_s/8}.
\end{align*}
To bound $(**)$, we note that $N_s(b) \ge M_s(b)$ and apply \Cref{fact:Hoeffding_maximal} on arms $a, a^*$ with the filtration $\Gcal$ to get:
\begin{align*}
    \Pr\br*{B_{a,s} \cap D_{a,s}} + \Pr\br*{B_{a^*,s} \cap D_{a^*,s}} \le 2e^{-m_s\Delta_a^2/4} + 2e^{-m_s\Delta_a^2/4} = 4e^{-m_s\Delta_a^2/4}.
\end{align*}
Combining these results, we get
\begin{align*}
    \Pr\br*{g_s = a} \le 2e^{-m_s/8} + 4e^{-m_s\Delta_a^2/4} \le 6e^{-m_s\Delta_a^2/8},
\end{align*}
where the last inequality is since for $\Delta_a \le 1$, we have that $m_s\Delta_a^2/8 \le m_s/8$ and $m_s\Delta_a^2/8 \le m_s\Delta_a^2/4$.
Recalling that $m_s = \frac{\epsilon(s-1)}{4K}$, we get
\begin{align*}
    \Pr\br*{g_s = a} \le 6e^{-\frac{\epsilon(s-1)\Delta_a^2}{32K}},
\end{align*}
as desired.
\end{proof}
Combining \Cref{lem:known_epsilon_uniform_bound_prob,lem:known_epsilon_decaying_bound_prob}, we get
\begin{align*}
    \Pr\br*{g_s = a} \le \min\brc*{q, 6e^{-\frac{\epsilon(s-1)\Delta_a^2}{32K}}}.
\end{align*}
Applying the union bound and taking the above bound, we get
{\small
\begin{align*}
     \E\brs*{\sum_{t=T_{SE} + 1}^{\infty}\Ind{\Delta_{\pi_t}>\epsilon}}
     \le \E\brs*{\sum_{s=1}^{\infty}\Ind{\Delta_{g_s} > \frac{3\epsilon}{4}}}
     \le \sum_{s=1}^{\infty}\sum_{a:\Delta_a > \frac{3\epsilon}{4}}\Pr\br*{g_s = a}
     \le \sum_{a:\Delta_a > \frac{3\epsilon}{4}}\sum_{s=1}^{\infty}6\min\brc*{q, e^{-\frac{\epsilon(s-1)\Delta_a^2}{32K}}}.
\end{align*}}
We bound the inner sum using the following lemma.
\begin{restatable}{lemma}{knownEpsilonMinSumBound}\label{lem:min_constant_exponential_sum}
    Let $0< q < 1$ and $\lambda\in(0,1]$. Then
\begin{align*}
    \sum_{j=0}^{\infty}\min\brc*{q,e^{-\lambda j}}
    \le q+\frac{q}{\lambda}\br*{\ln\br*{\frac{1}{q}}+2}.
\end{align*}
\end{restatable}
Proof of this lemma can be found in \Cref{subsec:known_epsilon_lemmas}.
We use \cref{lem:min_constant_exponential_sum} with $q = \frac{\epsilon^2}{16K}$ and $\lambda = \frac{\epsilon\Delta_a^2}{32K}$ to get
\begin{align*}
    \sum_{s=1}^{\infty}\min\brc*{q, e^{-\frac{\epsilon(s-1)\Delta_a^2}{32K}}}
    \le \frac{\epsilon^2}{16K} + \frac{2\epsilon}{\Delta_a^2}\br*{\ln\br*{\frac{16K}{\epsilon^2}} + 2},
\end{align*}
which in turn yields
\begin{align*}
    \sum_{a:\Delta_a > \frac{3\epsilon}{4}}\sum_{s=1}^{\infty}6\min\brc*{q, e^{-\frac{\epsilon(s-1)\Delta_a^2}{32K}}}
    \le \sum_{a:\Delta_a > \frac{3\epsilon}{4}} \frac{6\epsilon^2}{16K} + \frac{12\epsilon}{\Delta_a^2}\br*{\ln\br*{\frac{16K}{\epsilon^2}} + 2}.
\end{align*}
Since $\epsilon < 1$, we have that
\begin{align*}
    \sum_{a:\Delta_a > \frac{3\epsilon}{4}} \frac{6\epsilon^2}{16K} + \frac{12\epsilon}{\Delta_a^2}\br*{\ln\br*{\frac{16K}{\epsilon^2}} + 2}
    = \Ocal\br*{\sum_{a:\Delta_a > \frac{3\epsilon}{4}} \frac{1}{\Delta_a^2}{\ln\br*{\frac{K}{\epsilon^2}}}}.
\end{align*}
Combining this with \Cref{eq:known_epsilon_SE_bound}, we get that
\begin{align*}
    \E\brs*{N_{\epsilon}(\pi)} 
    &= \E\brs*{\sum_{t=1}^{T_{SE}}\Ind{\Delta_{\pi_t}>\epsilon}}
    + \E\brs*{\sum_{t=T_{SE} + 1}^{\infty}\Ind{\Delta_{\pi_t}>\epsilon}}
    = \Ocal\br*{\sum_{a:\Delta_a > \frac{3\epsilon}{4}} \frac{1}{\Delta_a^2}{\ln\br*{\frac{K}{\epsilon^2}}}},
\end{align*}
proving the desired bound.
\end{proof}
\newpage
\subsection{Technical Lemmas for the Proof of Theorem \ref{thm:known_epsilon_upper_bound}}\label{subsec:known_epsilon_lemmas}
\BetaDecreasingInn*
\begin{proof}
    Recall that $\beta_n = \sqrt{\frac{1}{2n}\ln\br*{\frac{64K^2n^2}{\epsilon^2}}}$. Since $x \mapsto \sqrt{x}$ is monotone increasing for $x > 0$, it is enough to show that $f(x) \triangleq \frac{1}{2x}\ln\br*{\frac{64K^2x^2}{\epsilon^2}}$ is monotone decreasing for $x \ge 1$. It holds that
    \begin{align*}
    f'(x)
    =-\frac{1}{2x^2}\ln\br*{\frac{64K^2x^2}{\epsilon^2}}+\frac{1}{x^2}
    = \frac{2-\ln\br*{\frac{64K^2x^2}{\epsilon^2}}}{2x^2}.
\end{align*}
Noting that for $K \ge 1$, $\epsilon \in (0,1)$ and $x \ge 1$ we have that $\frac{64K^2x^2}{\epsilon^2} \ge 64$, we get
\begin{align*}
    \ln\br*{\frac{64K^2x^2}{\epsilon^2}} \ge \ln(64)> 2,
\end{align*}
and so $f'(x) < 0$ for all $x \ge 1$. This implies that $f(x)$ is strictly monotone decreasing for $x \ge 1$, which in turn gives us that $\beta_n$ is monotone decreasing in $n$.
\end{proof}

\SEHorizonBound*
\begin{proof}
Let $x\in(0,1]$, and define $L \triangleq \ln\br*{\frac{1024K}{\epsilon x}}$ and $n_0 \triangleq \frac{2048L}{x^2}$.
Since $K\ge1$, $\epsilon\le1$ and $x\le1$, we have
$L\ge\ln 1024$.
Furthermore, since $\frac{4K}{q}=\frac{64K^2}{\epsilon^2}$, we have that
\begin{align*}
    \ln\br*{\frac{4Kn_0^2}{q}}
    =\ln\br*{\frac{64K^2n_0^2}{\epsilon^2}}
    =\ln 64+2\ln\br*{\frac{K}{\epsilon}}+2\ln 2048+2\ln L+4\ln\br*{\frac{1}{x}}.
    \end{align*}
By the definition of $L$,
\begin{align*}
    \ln\br*{\frac{K}{\epsilon}}\le L, \qquad \ln\br*{\frac{1}{x}}\le L, \qquad \ln L\le L.
\end{align*}
Furthermore, since $L\ge\ln 1024$, it holds that $\ln 64+2\ln 2048\le3L$, and so we get that $\ln\br*{\frac{4Kn_0^2}{q}} \le 11L$.
By the definition of $\beta_n$ and \Cref{lem:Beta_decreasing}, it follows that
\begin{align*}
    \beta_{\ceil{n_0}}^2 
    \le \frac{1}{2n_0}\ln\br*{\frac{4Kn_0^2}{q}}
    \le\frac{x^2}{4096L}\cdot 11L
    = \frac{11x^2}{4096}
    < \frac{x^2}{64},
\end{align*}
thus $\beta_{\ceil{n_0}} <\frac{x}{8}$. It follows that $n(x) \le \ceil{n_0}$, which proves the first result. 
Applying the first bound with $x=\epsilon/4$ gives
\begin{align*}
    n_\epsilon \le \ceil*{\frac{32768}{\epsilon^2}\ln\br*{\frac{4096K}{\epsilon^2}}},
\end{align*}
as desired.
\end{proof}

\knownEpsilonMinSumBound*
\begin{proof}
Let $J = \ceil*{\frac{1}{\lambda}\ln\br*{\frac{1}{q}}}$. Then
\begin{align*}
    \sum_{j=0}^{J-1}\min\brc*{q,e^{-\lambda j}}\le qJ\le q+\frac{q}{\lambda}\ln\br*{\frac{1}{q}}.
\end{align*}
Moreover,
\begin{align*}
    \sum_{j=J}^{\infty}\min\brc*{q,e^{-\lambda j}}\le \sum_{j=J}^{\infty}e^{-\lambda j}=\frac{e^{-\lambda J}}{1-e^{-\lambda}}\le\frac{q}{1-e^{-\lambda}}.
\end{align*}
Since $\lambda \in (0, 1]$, we have that $1-e^{-\lambda}\ge\frac{\lambda}{2}$, implying that
\begin{align*}
    \sum_{j=J}^{\infty}e^{-\lambda j} \le \frac{2q}{\lambda}.
\end{align*}
Combining the two inequalities yields
\begin{align*}
     \sum_{j=0}^{\infty}\min\brc*{q,e^{-\lambda j}} \le q+\frac{q}{\lambda}\ln\br*{\frac{1}{q}} + \frac{2q}{\lambda} = q + \frac{q}{\lambda}\br*{\ln\br*{\frac{1}{q}} + 2},
\end{align*}
giving us the desired bound.
\end{proof}

\newpage

\subsection{Proof of Proposition \ref{thm:lower_bound_delta_and_epsilon}}\label{appendix:lower_bounds}
\MinMaxLowerBound*
\begin{proof}
Fix some policy $\pi$. Let $T \in \N$ be some constant to be chosen later, and denote $N_\epsilon \br*{\pi, \nu, T} = \sum_{t=1}^{T} \Ind{\Delta_{\pi_t} > \epsilon}$.
We construct 2 instances, $\nu, \nu' \in \Ecal_{\delta}$, and show that $\pi$ must incur high expected sample complexity on at least one of them.

Let $\Delta = \max \brc*{\delta, 3\epsilon}$, and note that $\Delta \le 1/8$. We define an instance from $\Ecal_{\delta}$ by a vector, where the $i^{th}$ entry represents the mean of the $i^{th}$ arm.
The first instance $\nu$ is defined by the mean vector:
\begin{align*}
    \mu_{\nu}(a) 
= \begin{cases}
\frac{1}{2} + \Delta & a = 1 \\
\frac{1}{2} & a \ne 1
\end{cases}.
\end{align*}
Now, let $i = \argmin_{j>1} \E_{\nu} \brs*{n_{T+1} \br*{j}}$ be the suboptimal arm in $\nu$ that the policy $\pi$ played the least in expectation when running on horizon $T$.
We define the instance $\nu'$ by the following mean vector:
\begin{align*}
    \mu_{\nu'}(a) 
= \begin{cases}
\frac{1}{2} + \Delta & a = 1 \\
\frac{1}{2} + 2\Delta & a = i \\
\frac{1}{2} & a \notin \brc*{1, i}.
\end{cases}
\end{align*}
Consider the random variable $S_T = \sum_{t=1}^{T} \Ind{\pi_t \br{1} \ge \frac{\Delta - \epsilon}{\Delta}}$.
Note that by our choice of $\Delta$, on instance $\nu$ we have that $\Delta_{\pi_t} = \br*{1-\pi_t(1)}\Delta$, and so 
\begin{align*}
    N_{\epsilon} \br*{\pi, \nu, T} 
    = \sum_{t=1}^{T} \Ind{\br*{1-\pi_t(1)}\Delta > \epsilon}
    = \sum_{t=1}^{T} \Ind{\pi_t(1) < \frac{\Delta - \epsilon}{\Delta}}
    = T - S_T,
\end{align*}
implying that
\begin{align*}
    N_{\epsilon} \br*{\pi, \nu, T} \ge \Ind{S_T \le \frac{T}{2}} \cdot \frac{T}{2}.
\end{align*}
Taking the expectation on both sides yields:
\begin{align}\label{eq:lower_bound_first_event}
    \E_{\nu} \brs*{N_{\epsilon} \br*{\pi, \nu, T}} \ge \Pr\nolimits_{\nu} \br*{S_T \le \frac{T}{2}} \cdot \frac{T}{2}.
\end{align}
Moreover, by the definition of $\nu'$, we have that
\begin{align*}
    N_{\epsilon} \br*{\pi, \nu', T} 
    \ge \sum_{t=1}^{T} \Ind{\pi_t(1)\Delta > \epsilon}
    \ge \sum_{t=1}^{T} \Ind{\pi_t(1)> \frac{\epsilon}{\Delta}}
    \ge \sum_{t=1}^{T} \Ind{\pi_t(1) \ge \frac{\Delta - \epsilon}{\Delta}},
\end{align*}
where the last inequality is since $\frac{\Delta - \epsilon}{\Delta} > \frac{\epsilon}{\Delta}$ for our choice of $\Delta$.
The above implies that $N_{\epsilon} \br*{\pi, \nu', T} \ge S_T$, and so we get
\begin{align*}
    N_{\epsilon} \br*{\pi, \nu', T} \ge \Ind{S_T > \frac{T}{2}} \cdot \frac{T}{2},
\end{align*}
and taking the expectation on both sides yields:
\begin{align}\label{eq:lower_bound_second_event}
    \E_{\nu'} \brs*{N_{\epsilon} \br*{\pi, \nu', T}} \ge \Pr\nolimits_{\nu'} \br*{S_T > \frac{T}{2}} \cdot \frac{T}{2}.
\end{align}
Combining \Cref{eq:lower_bound_first_event,eq:lower_bound_second_event}, we have that:
\begin{align*}
    2 \max \brc*{\E_{\nu} \brs*{N_{\epsilon} \br*{\pi, \nu, T}},\E_{\nu'} \brs*{N_{\epsilon} \br*{\pi, \nu', T}}}
    &\ge \E_{\nu} \brs*{N_{\epsilon} \br*{\pi, \nu, T}} + \E_{\nu'} \brs*{N_{\epsilon} \br*{\pi, \nu', T}}\\
    &\ge \Pr\nolimits_{\nu} \br*{S_T \le \frac{T}{2}} \cdot \frac{T}{2} + \Pr\nolimits_{\nu'} \br*{S_T > \frac{T}{2}} \cdot \frac{T}{2}\\
    &= \br*{\Pr\nolimits_{\nu} \br*{S_T \le \frac{T}{2}} + \Pr\nolimits_{\nu'} \br*{S_T > \frac{T}{2}}}\cdot\frac{T}{2}.
\end{align*}
Applying the Bretagnolle-Huber inequality (Theorem 14.2 in \cite{BanditAlgsTor2017}) to the above yields:
\begin{align}\label{eq:lower_bound_prob_bound}
    \br*{\Pr\nolimits_{\nu} \br*{S_T \le \frac{T}{2}} + \Pr\nolimits_{\nu'} \br*{S_T > \frac{T}{2}}}\cdot\frac{T}{2} 
    \ge \br*{\frac{1}{2} \exp \br*{-D \br*{\Pr\nolimits_{\nu}, \Pr\nolimits_{\nu'}}}}\cdot\frac{T}{2}.
\end{align}
We use the divergence decomposition for bandits (Lemma 15.1 in \cite{BanditAlgsTor2017}) to get:
\begin{align*}
    D \br*{\Pr\nolimits_{\nu}, \Pr\nolimits_{\nu'}} 
    = \sum_{a \in \brs*{K}} \E_{\nu} \brs*{n_{T+1} (a)} D \br*{\nu_a, \nu'_a} 
    = \E_{\nu} \brs*{n_{T+1} (i)} D \br*{Ber \br*{1/2}, Ber \br*{1/2 + 2 \Delta}}.
\end{align*}
Since $\sum_{a \in \brs*{K}\setminus \brc*{1}} n_{T+1} (a) \le T$, it must hold that 
$\E_{\nu} \brs*{n_{T+1} \br*{i}} \le \frac{T}{K - 1}$. 
We combine this, and the bound on the divergence of Bernoulli RVs (\Cref{fact:Ber_KL_bound}) to get:
\begin{align*}
    \E_{\nu} \brs*{n_{T+1} (i)} D \br*{Ber \br*{1/2}, Ber \br*{1/2 + 2 \Delta}} \le \frac{32T\Delta^2}{3(K - 1)} \le \frac{11T\Delta^2}{K - 1}.
\end{align*}
We substitute this back to get:
\begin{align*}
    \br*{\frac{1}{2} \exp \br*{-D \br*{\Pr_{\nu}, \Pr_{\nu'}}}}\cdot\frac{T}{2}
    &\ge \frac{T}{4} \exp \br*{-\frac{11T\Delta^2}{K - 1}}.
\end{align*}

Denote by $x \triangleq \argmax_{T\in \R_+}\brc*{\frac{T}{4} \exp \br*{-\frac{11T\Delta^2}{K - 1}}} = \frac{K - 1}{11\Delta^2}$ the continuous maximizer of the lower bound. 
Since $\Delta \le \frac{1}{8}$, we have that $x \ge \frac{64(K-1)}{11} \ge 1$. Taking $T^* = \floor*{x} \in \N$ gives 
\begin{align*}
    f(\floor*{x}) 
    = \frac{\floor*{x}}{4} \exp \br*{-\frac{11\floor*{x}\Delta^2}{K - 1}}
    \ge \frac{\floor*{x}}{4} \exp \br*{-\frac{\floor*{x}}{x}}.
\end{align*}
Noting that $\frac{\floor*{x}}{x} \le 1$, and that $\floor*{x} \ge \frac{x}{2}$ for $x \ge 1$, we get
\begin{align*}
    \frac{\floor*{x}}{4} \exp \br*{-\frac{\floor*{x}}{x}}
    \ge \frac{x}{8} \exp \br*{-1}
    = \frac{K - 1}{88e\Delta^2}.
\end{align*}
Combining this with \Cref{eq:lower_bound_prob_bound} and the definition of $\Delta$ yields
\begin{align*}
    \max \brc*{\E_{\nu} \brs*{N_{\epsilon} \br*{\pi, \nu, T}},\E_{\nu'} \brs*{N_{\epsilon} \br*{\pi, \nu', T}}}
    \ge \frac{K - 1}{176e\Delta^2}
    = \frac{K - 1}{176e\max \brc*{\delta, 3\epsilon}^2}
    \ge \frac{K - 1}{1584e\max \brc*{\delta, \epsilon}^2}.
\end{align*}
Noting that $N_{\epsilon} \br*{\pi, \nu, T} \le N_{\epsilon} \br*{\pi, \nu, T + 1}$ for all $\pi, \nu$ and $T$ implies that for all $T \in \N$,
\begin{align*}
    \E_{\nu} \brs*{N_{\epsilon} \br*{\pi, \nu, T}} \le \E_{\nu} \brs*{N_{\epsilon} \br*{\pi, \nu}}.
\end{align*}
Finally, taking $T = T^*$ gives 
\begin{align*}
    \max \brc*{\E_{\nu} \brs*{N_{\epsilon} \br*{\pi, \nu}},\E_{\nu'} \brs*{N_{\epsilon} \br*{\pi, \nu'}}} 
    \ge \frac{K-1}{1584e \max\brc*{\delta, \epsilon}^2},
\end{align*}
which yields desired result.
\end{proof}

\newpage
\section{Uniform Algorithms -- Omitted Proofs}
\subsection{Proof of Proposition \ref{thm:U-ACE_and_consistent_lower_bound}}\label{appendix:U-ACE_and_consistent_lower_bound}
\UACEAndConsistentLowerBound*
\begin{proof}
    Denote $C_{\nu} = \sum_{a:\Delta_a > 0}\frac{3}{8\Delta_a}$, and suppose towards contradiction that 
    \begin{align*}
        \limsup_{\epsilon\downarrow 0} \epsilon \cdot \E\brs*{N_{\epsilon}\br*{\pi, \nu}} < C_{\nu}.
    \end{align*}
    If $C_{\nu} = 0$, the bound holds trivially. If $C_{\nu} > 0$ then the above implies that there exists $c < C_{\nu}$ and $\epsilon_0 > 0$ such that for all $\epsilon \in \br*{0, \epsilon_0}$:
    \begin{align*}
        \epsilon \cdot \E\brs*{N_{\epsilon}\br*{\pi, \nu}} < c.
    \end{align*}
    Using the above and recalling the following decomposition from the proof of \Cref{thm:UACE_implies_expected_regret}:
    \begin{align*}
    R_T(\pi) 
    \le \int_{\epsilon=0}^{1}\min\brc*{T, \E\brs*{N_{\epsilon}(\pi)}}d\epsilon,
    \end{align*}
    gives us that
    \begin{align*}
        R_T(\pi) 
        \le \int_{\epsilon=0}^{1}\min\brc*{T, \E\brs*{N_{\epsilon}(\pi)}}d\epsilon
        \le \int_{\epsilon=0}^{\epsilon_0}\min\brc*{T, \frac{c}{\epsilon}}d\epsilon
        + \int_{\epsilon=\epsilon_0}^{1}\min\brc*{T, \E\brs*{N_{\epsilon_0}(\pi)}}d\epsilon.
    \end{align*}
    By the monotonicity of the integral:
    \begin{align*}
        \int_{\epsilon=\epsilon_0}^{1}\min\brc*{T, \E\brs*{N_{\epsilon_0}(\pi)}}d\epsilon
        \le \E\brs*{N_{\epsilon_0}(\pi)},
    \end{align*}
    and so this term is a constant with respect to $T$.
    Note that for all sufficiently large $T$, namely $T > \frac{c}{\epsilon_0}$, we have that
    \begin{align*}
        \int_{\epsilon=0}^{\epsilon_0}\min\brc*{T, \frac{c}{\epsilon}}d\epsilon
        \le \int_{\epsilon=0}^{c/T}Td\epsilon + \int_{\epsilon=c/T}^{\epsilon_0}\frac{c}{\epsilon}d\epsilon
        = c + c\ln\br*{\frac{\epsilon_0 T}{c}}
        = c(1 + \ln(\epsilon_0) + \ln(1/c)) + c\ln(T),
    \end{align*}
    and so since $\epsilon_0$ and $c$ are constants, we have that
    \begin{align*}
        \limsup_{T\rightarrow\infty}\frac{R_T(\pi)}{\ln(T)} \le c.
    \end{align*}
    This contradicts the Lai-Robbins lower bound (\Cref{fact:Lai-Robbins_lower_bound}), thus implying our desired result.
\end{proof}

\newpage

\subsection{Proof of Proposition \ref{thm:U-ACE_without_consistency_lower_bound}}\label{appendix:U-ACE_without_consistency_lower_bound}
\UACEWithoutConsistencyLowerBound*
\begin{proof}
     Let $\nu = \br*{Ber\br*{\frac{1}{2}}, Ber\br*{\frac{1}{4}}}$, and let $\nu' = \br*{Ber\br*{\frac{1}{2}}, Ber\br*{\frac{3}{4}}}$. 
     We prove the following:
     \begin{lemma}\label{lem:good_in_one_implies_bad_in_other_lower_bound}
         If $\Pr\nolimits_{\nu'}\br*{n_{\infty}(2) = \infty} = 1$, then $\Pr\nolimits_{\nu}\br*{n_{\infty}(2) = \infty} = 1$.
     \end{lemma}
     \begin{proof} 
     Assume that $\Pr\nolimits_{\nu'}\br*{n_{\infty}(2) = \infty} = 1$. We utilize the following lemma:
     \begin{lemma}[Lemma $1$ from \cite{kaufmann2016complexity}]\label{lem:KL_lower_bound_bandits_kaufmann}
         Let $\nu, \nu'$ be two bandit instances with $K$ arms such that for all $a$, the distributions of $\nu_a$ and $\nu_a'$ are mutually absolutely continuous. Then for any almost-surely finite stopping time $\sigma$ with respect to filtration $(\F_t)$:
         \begin{align*}
             \sum_{a\in\brs*{K}}\E_{\nu}\brs*{n_{\sigma + 1}(a)}D\br*{\nu_a, \nu_a'} \ge \sup_{A\in \F_{\sigma}}d\br*{\Pr\nolimits_{\nu}(A), \Pr\nolimits_{\nu'}(A)},
         \end{align*}
         where $d(x,y) \triangleq x\ln\br*{x/y} + (1-x)\ln\br*{(1-x)/(1-y)}$. 
     \end{lemma}
     Note that $\nu_1=\nu_1'$, while $\nu_2 = Ber(1/4)$ and $\nu_2' = Ber(3/4)$ are mutually absolutely
     continuous, since both assign positive probability to $0$ and $1$. Moreover,
     \begin{align*}
         D\br*{\nu_2,\nu_2'} = D\br*{Ber(1/4),Ber(3/4)} = \frac{\ln(3)}{2}.
     \end{align*}
     Fix $n \in \N$, and define $\tau_n \triangleq \inf \brc*{t : n_t(2) = n}$. Let $T \in \N$, and define $\sigma_{n,T} \triangleq \tau_n \wedge T$ to be an almost-surely finite stopping time.
     We look at the event $E_{n,T} \triangleq \brc*{\tau_n \le T} = \brc*{n_T(2) \ge n}$. We can see that $E_{n,T}$ is in $\F_{\sigma_{n,T}}$, which implies by \Cref{lem:KL_lower_bound_bandits_kaufmann} that
     \begin{align*}
         \E_{\nu}\brs*{n_{\sigma_{n,T} + 1}(2)}D\br*{\nu_2, \nu_2'} \ge d\br*{\Pr\nolimits_{\nu}(E_{n,T}), \Pr\nolimits_{\nu'}(E_{n,T})},
     \end{align*}
     but by the definition of $\sigma_{n,T}$, we have that $n_{\sigma_{n,T} + 1}(2) \le n + 1$, which implies that $d\br*{\Pr\nolimits_{\nu}(E_{n,T}), \Pr\nolimits_{\nu'}(E_{n,T})} \le \frac{(n + 1)\ln(3)}{2} = C_n$.
     Since $\Pr\nolimits_{\nu'}\br*{n_{\infty}(2) = \infty} = 1$, we have that $\Pr\nolimits_{\nu'}(E_{n,T}) \rightarrow 1$ as $T \rightarrow \infty$. 

     Now, denote $p_{n,T} \triangleq \Pr\nolimits_{\nu}(E_{n,T})$, $q_{n,T} \triangleq \Pr\nolimits_{\nu'}(E_{n,T})$, and note that the sequence of events $E_{n,T}$ is increasing in $T$, thus implying that $p_{n,T} \uparrow p_n = \Pr\nolimits_{\nu}\br*{n_\infty(2) \ge n}$, and that $q_{n,T} \uparrow \Pr\nolimits_{\nu'}\br*{n_\infty(2) \ge n} = 1$ by our assumption. 
     Suppose towards contradiction that $p_n < 1$, and let $\eta = \frac{1- p_n}{2} > 0$.
     Since $p_{n,T} \uparrow p_n$, we have that $p_{n,T} \le 1 - \eta$, which gives us $1 - p_{n,T} \ge \eta$. 
     Recall the bound $d(p,q) \ge (1-p)\ln(1/(1-q)) - \ln(2)$, which implies that
     \begin{align*}
         d(p_{n,T},q_{n,T}) 
         \ge (1-p_{n,T})\ln(1/(1-q_{n,T})) - \ln(2)
         \ge \eta\ln(1/(1-q_{n,T})) - \ln(2).
     \end{align*}
     But since $q_{n,T} \rightarrow 1$, we have that $\eta\ln(1/(1-q_{n,T})) - \ln(2) \rightarrow \infty$, contradicting the fact that $d(p_{n,T},q_{n,T})$ is bounded by $C_n$ for all $T$.
    This in turn implies that $p_{n,T} \uparrow 1$, or equivalently:
    \begin{align*}
        \forall n\in \N : \quad\Pr\nolimits_{\nu}\br*{n_\infty(2) \ge n} = 1.
    \end{align*}
    Note that the sequence $A_n = \brc*{n_{\infty}(2) \ge n}$ is decreasing in $n$, thus by continuity and the above
    \begin{align*}
        \Pr\nolimits_{\nu}\br*{n_\infty(2) = \infty} 
        = \Pr\nolimits_{\nu}\br*{\bigcap_{n=1}^{\infty}A_n}
        = \lim_{n \rightarrow \infty}\Pr\nolimits_{\nu}\br*{A_n}
        = 1,
    \end{align*}
    as desired.
    \end{proof}
    We continue by showing that on at least one of the instances $\nu, \nu'$, we have that $\E\brs*{\sum_{t=1}^{\infty} \Delta_{\pi_t}} = \infty$.
    Notice it suffices to show that either $\Pr\nolimits_{\nu}\br*{n_\infty(2) = \infty} > 0$ or $\Pr\nolimits_{\nu'}\br*{n_\infty(1) = \infty} > 0$, since this would imply the above by applying \Cref{thm:expected_regret_equivalent_for_policy_gap} for all finite $T$ and MCT.
    
    As such, if $\Pr\nolimits_{\nu'}\br*{n_\infty(1) = \infty} > 0$, we are done, and if $\Pr\nolimits_{\nu'}\br*{n_\infty(1) = \infty} = 0$, then $\Pr\nolimits_{\nu'}\br*{n_\infty(2) = \infty} = 1$, which by \Cref{lem:good_in_one_implies_bad_in_other_lower_bound} implies that $\Pr\nolimits_{\nu}\br*{n_\infty(2) = \infty} = 1 > 0$.

    This shows that there exists an instance such that $\E\brs*{\sum_{t=1}^{\infty} \Delta_{\pi_t}} = \infty$, which implies that on that instance
    \begin{align*}
        \int_{0}^{\Delta_{\max}}\E\brs*{N_{\epsilon}(\pi)}d\epsilon
        = \int_{0}^{\Delta_{\max}} \E\brs*{\sum_{t=1}^{\infty}\Ind{\Delta_{\pi_t} > \epsilon}} d\epsilon
        = \E\brs*{\sum_{t=1}^{\infty} \int_{0}^{\Delta_{\max}}\Ind{\Delta_{\pi_t} > \epsilon} d\epsilon}
        = \E\brs*{\sum_{t=1}^{\infty} \Delta_{\pi_t}}
        = \infty.
    \end{align*}
    Fix some $\epsilon_0 \in (0, \Delta_{\max})$. $\pi$ is U-ACE, and so $\E\brs*{N_{\epsilon_0}(\pi)} < \infty$.
    Moreover, since $\E\brs*{N_{\epsilon}(\pi)}$ is non-increasing in $\epsilon$, we have that
    \begin{align*}
        \int_{\epsilon_0}^{\Delta_{\max}}\E\brs*{N_{\epsilon}(\pi)}d\epsilon \le \br*{\Delta_{\max} - \epsilon_0}\E\brs*{N_{\epsilon_0}(\pi)} < \infty
    \end{align*}
    Thus implying that $\int_{0}^{\epsilon_0}\E\brs*{N_{\epsilon}(\pi)}d\epsilon = \infty$. 
    Let $\alpha \in (0,1)$, and suppose towards contradiction there exists $C > 0$ such that for all $\epsilon \in (0,\epsilon_0)$:
    \begin{align*}
        \E\brs*{N_{\epsilon}(\pi)} \le \frac{C}{\epsilon^{\alpha}}.
    \end{align*}
    By the monotonicity of the integral and the fact that $\int_{0}^{\epsilon_0}\frac{C}{\epsilon^{\alpha}}d\epsilon < \infty$, we have that
    \begin{align*}
        \int_{0}^{\epsilon_0}\E\brs*{N_{\epsilon}(\pi)}d\epsilon
        \le  \int_{0}^{\epsilon_0}\frac{C}{\epsilon^{\alpha}}d\epsilon
        < \infty
    \end{align*}
    In contradiction to $\int_{0}^{\epsilon_0}\E\brs*{N_{\epsilon}(\pi)}d\epsilon = \infty$. Therefore, for all $C > 0$ there exists $\epsilon \in (0, \epsilon_0)$ such that
    \begin{align*}
        \E\brs*{N_{\epsilon}(\pi)} \ge \frac{C}{\epsilon^{\alpha}}.
    \end{align*}  
    The claim follows by noting that we can take $\epsilon_0 \downarrow 0$ and $C \rightarrow \infty$.
\end{proof}

\newpage

\subsection{Proof of Theorem \ref{thm:thompson_sampling_is_U-ACE}}\label{appendix:thompson_sampling_is_U-ACE}
\thompsonSamplingIsUACE*
\begin{proof}
Our proof is inspired by the methods used in \cite{AgrawalGoyalTS2017}. We adopt their choices of events and decomposition, but we differ in two main ways: 
\begin{enumerate}
    \item Analyzing the sample complexity requires different techniques than bounding the regret, since our objective measures the conditional distribution of the policy itself, thus requiring a more subtle approach. Moreover, our analysis is over an infinite horizon, requiring delicate tracking of some of the events, which is unnecessary when interacting for a finite horizon.

    \item In contrast to previous works, we cannot apply a monotonicity argument to reduce to a unique optimal arm. The intuition is that there is no guarantee that playing an optimal arm will conditionally decrease the policy gap at the next time-step, with an empirical counterexample given in \Cref{sec:optimal_arms_TS}. This adds an additional layer to our analysis, requiring us to design events that are agnostic to the number of optimal arms, and also to the optimal representative chosen in our proofs.
\end{enumerate}
Since $\pi^{TS(\Ncal)}$ is the only policy we deal with in this proof, we refer to it as $\pi$ for brevity. We denote 
\begin{align*}
    \Acal^* \triangleq \brc*{a \in \brs*{K} : \Delta_a = 0},
    \qquad
    \Acal_{sub} \triangleq \brc*{a \in \brs*{K} : \Delta_a > 0},
    \qquad
    M \triangleq \abs{\Acal_{sub}}.
\end{align*}
Note that if $M = 0$, then all arms are optimal and $N_{\epsilon}(\pi) = 0$ almost surely, giving us the desired result. We henceforth assume $M \ge 1$, and address the case of multiple optimal arms when needed. For every suboptimal arm $a\in \Acal_{sub}$, we define 
\begin{align*}
    x_a \triangleq \mu_a + \frac{\Delta_a}{3},
    \qquad
    y_a \triangleq \mu^* - \frac{\Delta_a}{3},
\end{align*}
and the events
\begin{align*}
    E_a^\mu(t) \triangleq \brc*{\hat\mu_t(a)\le x_a},
    \qquad
    E_a^\theta(t) \triangleq \brc*{\theta_t(a)\le y_a}.
\end{align*}
We prove the following:
\begin{lemma}\label{lem:esier_event_to_analyze_TSN}
    For all $t$, we have that
    \begin{align*}
        \Pr \br*{\Delta_{\pi_t} > \epsilon} 
        \le \sum_{a : \Delta_a \ge \frac{\epsilon}{M}} 
        \Pr \br*{\pi_t(a) > \frac{\epsilon}{M \Delta_a}}.
    \end{align*}
\end{lemma}
\begin{proof}
    Assume that $\Delta_{\pi_t} > \epsilon$, and assume for the sake of contradiction that $\pi_t(a) \le \frac{\epsilon}{M \cdot \Delta_a}$ for all $a \in \Acal_{sub}$. We get that
    $$\Delta_{\pi_t} 
    = \sum_{a\in \brs*{K}}{\pi_{t}\br*{a}\Delta_a} 
    \le \sum_{a\in \Acal_{sub}}{\frac{\epsilon}{M \Delta_a} \Delta_a}
    = \frac{M}{M} \cdot \epsilon
    = \epsilon,
    $$
    which is a contradiction, implying that $\brc*{\Delta_{\pi_t} > \epsilon} 
        \subseteq \bigcup_{a \in \Acal_{sub}}\brc*{\pi_t(a) > \frac{\epsilon}{M \Delta_a}}$. This gives us
    \begin{align*}
        \Pr \br*{\Delta_{\pi_t} > \epsilon} 
        \le \Pr \br*{\bigcup_{a \in \Acal_{sub}}\brc*{\pi_t(a) > \frac{\epsilon}{M \Delta_a}}}
        \le \sum_{a : \Delta_a \ge \frac{\epsilon}{M}} 
        \Pr \br*{\pi_t(a) > \frac{\epsilon}{M \Delta_a}},
    \end{align*}
    where the last inequality is by the union bound and the fact that for suboptimal arms with $\Delta_a < \epsilon/M$, we get that the event $\brc*{\pi_t(a) > \frac{\epsilon}{M \Delta_a}}$ is empty.
\end{proof}
We therefore have that 
{\small
\begin{align}\label{eq:TS_sample_complexity_decomposition}
    \E\brs*{\sum_{t=1}^{\infty}{\Ind{\Delta_{\pi_t}>\epsilon}}}
    \overset{(1)}{=} \sum_{t=1}^{\infty}{\Pr \br*{\Delta_{\pi_t}>\epsilon}} 
    \overset{(2)}{\le} \sum_{t=1}^{\infty}\sum_{\Delta_a \ge \frac{\epsilon}{M}}{\Pr \br*{\pi_t(a) > \frac{\epsilon}{M \Delta_a}}}
    \overset{(3)}{=} \sum_{\Delta_a \ge \frac{\epsilon}{M}}{\sum_{t=1}^{\infty}\Pr \br*{\pi_t(a) > \frac{\epsilon}{M \Delta_a}}},
\end{align}}
where $(1)$ is by MCT, $(2)$ is by \cref{lem:esier_event_to_analyze_TSN} and $(3)$ is by Fubini-Tonelli.
Since $\pi_t(a) = \Pr \br*{A_t = a \vert \F_{t-1}}$, we have that
\begin{align*}
\sum_{\Delta_a \ge \frac{\epsilon}{M}}{\sum_{t=1}^{\infty}\Pr \br*{\pi_t(a) > \frac{\epsilon}{M \Delta_a}}}
&= \sum_{\Delta_a \ge \frac{\epsilon}{M}}{\sum_{t=1}^{\infty}\Pr \br*{\Pr\br*{A_t = a \vert \F_{t-1}} > \frac{\epsilon}{M \Delta_a}}}.
\end{align*}
For the rest of our analysis, we fix an arm $a \in \Acal_{sub}$ with $\Delta_a \ge \epsilon/M$. 
By the law of total probability
\begin{align*}
    \Pr\br*{A_t = a \vert \F_{t-1}} 
    = \Pr\br*{A_t = a, \overline{E_{a}^{\mu}\br*{t}} \vert \F_{t-1}}
    + \Pr\br*{A_t = a, E_{a}^{\mu}\br*{t}, \overline{E_{a}^{\theta}\br*{t}} \vert \F_{t-1}} 
    + \Pr\br*{A_t = a, E_{a}^{\mu}\br*{t}, E_{a}^{\theta}\br*{t} \vert \F_{t-1}},
\end{align*}
which by the union bound implies that
{\small
\begin{align*}
    &\sum_{t=1}^{\infty}\Pr \br*{\pi_t(a) > \frac{\epsilon}{M \Delta_a}} \\
    &= \sum_{t=1}^{\infty}\Pr\br*{
    \Pr\br*{A_t = a, \overline{E_{a}^{\mu}\br*{t}} \vert \F_{t-1}}
    + \Pr\br*{A_t = a, E_{a}^{\mu}\br*{t}, \overline{E_{a}^{\theta}\br*{t}} \vert \F_{t-1}} 
    + \Pr\br*{A_t = a, E_{a}^{\mu}\br*{t}, E_{a}^{\theta}\br*{t} \vert \F_{t-1}} > \frac{\epsilon}{M \Delta_a}} \\
    &\le 
    \underbrace{\sum_{t=1}^{\infty}\Pr\br*{\Pr\br*{A_t = a, \overline{E_{a}^{\mu}\br*{t}} \vert \F_{t-1}} > \frac{\epsilon}{3M \Delta_a}}}_{S_1}
    + \underbrace{\sum_{t=1}^{\infty}\Pr\br*{\Pr\br*{A_t = a, E_{a}^{\mu}\br*{t}, \overline{E_{a}^{\theta}\br*{t}} \vert \F_{t-1}} > \frac{\epsilon}{3M \Delta_a}}}_{S_2} \\
    &\quad+ \underbrace{\sum_{t=1}^{\infty}\Pr\br*{\Pr\br*{A_t = a, E_{a}^{\mu}\br*{t}, E_{a}^{\theta}\br*{t} \vert \F_{t-1}} > \frac{\epsilon}{3M \Delta_a}}}_{S_3}.
\end{align*}
}
We bound each sum $S_1, S_2, S_3$ separately. We begin with $S_1$:
\begin{lemma}\label{lem:Bound_for_first_sum}
    It holds that:
    \begin{align*}
        \sum_{t=1}^{\infty}\Pr\br*{\Pr\br*{A_t = a, \overline{E_{a}^{\mu}\br*{t}} \vert \F_{t-1}} > \frac{\epsilon}{3M \Delta_a}}
        \le \frac{33M}{2\epsilon \Delta_a}.
    \end{align*}
\end{lemma}
\begin{proof}
By Markov's inequality, we have that:
{\small
\begin{align*}
    \Pr\br*{\Pr\br*{A_t = a, \overline{E_{a}^{\mu}\br*{t}} \vert \F_{t-1}} > \frac{\epsilon}{3M \Delta_a}}
    \le \frac{\E_{\F_{t-1}}\brs*{\Pr\br*{A_t = a, \overline{E_{a}^{\mu}\br*{t}} \vert \F_{t-1}}}}{\epsilon/3M \Delta_a}
    = \frac{3M\Delta_a\E_{\F_{t-1}}\brs*{\Pr\br*{A_t = a, \overline{E_{a}^{\mu}\br*{t}} \vert \F_{t-1}}}}{\epsilon},
\end{align*}
}
where $\E_{\F_t}\brs*{\cdot}$ denotes the expectation over the randomness of the history.
By the tower rule:
{\small
\begin{align*}
    \sum_{t=1}^{\infty}\Pr\br*{\Pr\br*{A_t = a, \overline{E_{a}^{\mu}\br*{t}} \vert \F_{t-1}} > \frac{\epsilon}{3M \Delta_a}}
    \le \sum_{t=1}^{\infty}\frac{3M\Delta_a\E_{\F_{t-1}}\brs*{\Pr\br*{A_t = a, \overline{E_{a}^{\mu}\br*{t}} \vert \F_{t-1}}}}{\epsilon} 
    &= \frac{3M\Delta_a}{\epsilon}\sum_{t=1}^{\infty}
    \Pr \br*{A_t = a, \overline{E_{a}^{\mu}\br*{t}}}.
\end{align*}
}
We analyze $\sum_{t=1}^{\infty}{\Pr \br*{A_t = a, \overline{E_{a}^{\mu}\br*{t}}}}$. 
Note that $\overline{E_{a}^{\mu}\br*{t}} = \brc*{\hat{\mu}_t (a) > x_a}$, an event dependent only on $\hat{\mu}_t(a)$, which changes only when arm $a$ is pulled. This implies that
\begin{align*}
    \sum_{t=1}^{\infty}{\Pr \br*{A_t = a, \overline{E_{a}^{\mu}\br*{t}}}}
    = \E\brs*{\sum_{t=1}^{\infty}{\Ind{A_t = a}\Ind{\hat{\mu}_t (a) > x_a}}} 
    \overset{(*)}{\le} \E\brs*{\sum_{k=0}^{\infty}\Ind{\tilde{\mu}_{a,k} > x_a}} 
    \le 1 + \sum_{k=1}^{\infty}\Pr \br*{\tilde{\mu}_{a,k} > x_a},
\end{align*}
where $(*)$ is since $n_{\infty}(a) \le \infty$, and on the $(k+1)^{th}$ pull of arm $a$, the empirical mean is exactly $\frac{k}{k+1}\tilde{\mu}_{a,k} \le \tilde{\mu}_{a,k}$.   
Applying Hoeffding's inequality (\Cref{fact:hoeffding_inequality}) and \Cref{fact:inequality_for_infinite_sums} yields
\begin{align*}
    1 + \sum_{k=1}^{\infty}\Pr \br*{\tilde{\mu}_{a,k} > x_a} 
    \le 1 + \sum_{k=1}^{\infty}\exp\br*{-2k\br*{\Delta_a/3}^2}
    = 1 + \frac{e^{-2\Delta_a^2/9}}{1 - e^{-2\Delta_a^2/9}}
    \le 1 + \frac{9}{2\Delta_a^2}
    \le \frac{11}{2\Delta_a^2},
\end{align*}
where the last inequality is since $\Delta_a \le 1$. Combining, we get that
{\small
\begin{align*}
    \sum_{t=1}^{\infty}\Pr\br*{\Pr\br*{A_t = a, \overline{E_{a}^{\mu}\br*{t}} \vert \F_{t-1}} > \frac{\epsilon}{3M \Delta_a}} 
    \le \frac{3M\Delta_a}{\epsilon}\sum_{t=1}^{\infty}
    \Pr \br*{A_t = a, \overline{E_{a}^{\mu}\br*{t}}}
    \le \frac{3M\Delta_a}{\epsilon} \cdot \frac{11}{2\Delta_a^2}
    = \frac{33M}{2\epsilon\Delta_a}.
\end{align*}}
\end{proof}
To bound $S_2$, we show the following:
\begin{restatable}{lemma}{BoundForSecondSum}\label{lem:Bound_for_second_sum}
    It holds that:
    \begin{align*}
        \sum_{t=1}^{\infty}\Pr\br*{\Pr\br*{A_t = a, E_{a}^{\mu}\br*{t}, \overline{E_{a}^{\theta}\br*{t}} \vert \F_{t-1}} > \frac{\epsilon}{3M \Delta_a}}
        \le \frac{114 M}{\epsilon\Delta_a}\ln \br*{\frac{3M\Delta_a}{\epsilon}}.
    \end{align*}
\end{restatable}
\begin{proof}
    Let $L_a \triangleq {\frac{18}{\Delta_a^2}\ln \br*{\frac{3M\Delta_a}{\epsilon}}}$, and define
    \begin{align*}
        &Z_t^{<} \triangleq \Pr\br*{A_t = a, n_t(a) < L_a, E_{a}^{\mu}\br*{t}, \overline{E_{a}^{\theta}\br*{t}} \vert \F_{t-1}},\quad 
        Z_t^{\ge} \triangleq \Pr\br*{A_t = a, n_t(a) \ge L_a, E_{a}^{\mu}\br*{t}, \overline{E_{a}^{\theta}\br*{t}} \vert \F_{t-1}},
    \end{align*}
    noting that $\Pr\br*{A_t = a, E_{a}^{\mu}\br*{t}, \overline{E_{a}^{\theta}\br*{t}} \vert \F_{t-1}} = Z_t^{<} + Z_t^{\ge}$ for all $t \in \N$.
    Since both $n_t(a)$ and $\hat{\mu}_t(a)$ are $\F_{t-1}$-measurable, we have that
    \begin{align*}
        Z_t^{\ge} 
        &= \Pr\br*{A_t = a, n_t(a) \ge L_a, E_{a}^{\mu}\br*{t}, \overline{E_{a}^{\theta}\br*{t}} \vert \F_{t-1}}\\
        &= \E\brs*{\Ind{A_t = a}\Ind{n_t(a) \ge L_a} \Ind{\hat{\mu}_t(a) \le x_a}\Ind{\theta_t(a) > y_a} \vert \F_{t-1}}\\
        &\le \Ind{n_t(a) \ge L_a} \Ind{\hat{\mu}_t(a) \le x_a}\E\brs*{\Ind{\theta_t(a) > y_a} \vert \F_{t-1}}.
    \end{align*}
    Furthermore, on the event $E_{a}^{\mu}\br*{t}$, the conditional posterior sample 
    \begin{align*}
        \theta_t(a) \vert \F_{t-1} \sim \Ncal \br*{\hat{\mu}_t(a), \frac{1}{1 + n_t(a)}}
    \end{align*}
    is stochastically dominated by $\Ncal\br*{x_a, \frac{1}{1+n_t(a)}}$, and so by \Cref{fact:Gaussian_normalized_tail_bound} we get
    \begin{align*}
        \Ind{n_t(a) \ge L_a} \Ind{\hat{\mu}_t(a) \le x_a}\E\brs*{\Ind{\theta_t(a) > y_a} \vert \F_{t-1}} 
        &\le \Ind{n_t(a) \ge L_a}\Pr\br*{\Ncal\br*{x_a, \frac{1}{1+n_t(a)}} > y_a \vert \F_{t-1}}\\
        &\le \Pr\br*{\Ncal\br*{0, 1} > \sqrt{L_a + 1}\br*{y_a - x_a}}\\
        &\le \frac{1}{2}\exp\br*{\frac{-L_a\Delta_a^2}{18}}\\
        &= \frac{\epsilon}{6M\Delta_a}.
    \end{align*}
    The above shows that $Z_t^{\ge} \le \frac{\epsilon}{6M\Delta_a}$ for all $t\in\N$, which implies that
    \begin{align*}
        \brc*{\Pr\br*{A_t = a, E_{a}^{\mu}\br*{t}, \overline{E_{a}^{\theta}\br*{t}} \vert \F_{t-1}} > \frac{\epsilon}{3M\Delta_a}} 
        = \brc*{Z_t^{<} + Z_t^{\ge} > \frac{\epsilon}{3M\Delta_a}} 
        \subseteq \brc*{Z_t^{<} > \frac{\epsilon}{6M\Delta_a}},
    \end{align*}
    and as such
    \begin{align*}
        \sum_{t=1}^{\infty}\Pr\br*{\Pr\br*{A_t = a, E_{a}^{\mu}\br*{t}, \overline{E_{a}^{\theta}\br*{t}} \vert \F_{t-1}} > \frac{\epsilon}{3M \Delta_a}}
        \le \sum_{t=1}^{\infty}\Pr\br*{Z_t^{<} > \frac{\epsilon}{6M\Delta_a}}.
    \end{align*}
    By Markov's inequality:
    \begin{align*}
        \sum_{t=1}^{\infty}\Pr\br*{Z_t^{<} > \frac{\epsilon}{6M\Delta_a}}
        &\le \frac{6M\Delta_a}{\epsilon}\sum_{t=1}^{\infty}\E\brs*{Z_t^{<}}\\
        &\le \frac{6M\Delta_a}{\epsilon}\E\brs*{\sum_{t=1}^{\infty}\E\brs*{\Ind{A_t = a}\Ind{n_t(a) < L_a} \Ind{\hat{\mu}_t(a) \le x_a}\Ind{\theta_t(a) > y_a} \vert \F_{t-1}}}\\
        &\le \frac{6M\Delta_a}{\epsilon}\E\brs*{\ceil*{L_a}}\\
        &= \frac{6M\Delta_a}{\epsilon}\ceil*{L_a},
    \end{align*}
    where the last inequality is since for any history, the indicators $\Ind{A_t = a}\Ind{n_t(a) < L_a}$ can happen at most $\ceil*{L_a}$ times since each time $\Ind{A_t = a}$ occurs, $n_t(a)$ increases by $1$.
    Combining and substituting in the definition of $L_a$ yields
    \begin{align*}
        \sum_{t=1}^{\infty}\Pr\br*{\Pr\br*{A_t = a, E_{a}^{\mu}\br*{t}, \overline{E_{a}^{\theta}\br*{t}} \vert \F_{t-1}} > \frac{\epsilon}{3M \Delta_a}}
        \le \frac{6M\Delta_a}{\epsilon}\ceil*{L_a} \le \frac{114 M}{\epsilon\Delta_a}\ln \br*{\frac{3M\Delta_a}{\epsilon}},
    \end{align*}
    where the last inequality is since $\ln\br*{3M\Delta_a/\epsilon} \ge 1$ for $\Delta_a \ge \epsilon/M$.
\end{proof}
Finally, we bound $S_3$:
\begin{restatable}{lemma}{BoundForThirdSum}\label{lem:Bound_for_third_sum}
    There exists a universal constant $C > 0$ such that
    \begin{align*}
        \sum_{t=1}^{\infty}\Pr\br*{\Pr\br*{A_t = a, E_{a}^{\mu}\br*{t}, E_{a}^{\theta}\br*{t} \vert \F_{t-1}} > \frac{\epsilon}{3M \Delta_a}}
        \le \frac{CM}{\epsilon\Delta_a}.
    \end{align*}
\end{restatable}
\begin{proof}
    Fix an arbitrary optimal arm $a^* \in \Acal^*$, and define
    \begin{align*}
        p_{a,t} \triangleq \Pr\br*{\theta_{t}(a^*)>y_a\vert \F_{t-1}},
        \qquad
        q_{a,t}\triangleq\frac{1-p_{a,t}}{p_{a,t}}.
    \end{align*}
    \begin{lemma}\label{lem:TSN_prob_bound}
        For every round $t$, we have that
        \begin{align*}
            \Pr\br*{A_t = a, E_{a}^{\mu}\br*{t}, E_{a}^{\theta}\br*{t} \vert \F_{t-1}}
            \le q_{a,t}\Pr\br*{A_t = a^*, E_{a}^{\mu}\br*{t}, E_{a}^{\theta}\br*{t} \vert \F_{t-1}}.
        \end{align*}
    \end{lemma}
    \begin{proof}
        We begin by noting that $E_{a}^{\mu}\br*{t}$ is $\F_{t-1}$-measurable, and thus on any instantiation of $\F_{t-1}$ such that $E_{a}^{\mu}\br*{t}$ is false, the probability on the left-hand side is $0$ and thus the inequality holds trivially. We let $F_{t-1}$ be an instantiation of $\F_{t-1}$ on which $E_{a}^{\mu}\br*{t}$ is true, and define
        \begin{align*}
            B_{a,t} \triangleq \brc*{\theta_t(a) \le y_a} \cap \brc*{\theta_t(a) \ge \theta_t(b), \forall b\ne a^*}.
        \end{align*}        
        Note that $B_{a,t}$ is conditionally independent of $\theta_t(a^*)$ given $F_{t-1}$, and so by the definition of $E_{a}^{\theta}\br*{t}$:
        \begin{align*}
            \Pr\br*{A_t = a, E_{a}^{\mu}\br*{t}, E_{a}^{\theta}\br*{t} \vert \F_{t-1} = F_{t-1}}
            &\le \Pr\br*{\brc*{\theta_t(a) \le y_a}\cap\brc*{\forall b\in[K]: \theta_t(a) \ge \theta_t(b)} \vert \F_{t-1} = F_{t-1}}\\
            & = \Pr\br*{B_{a,t}, \theta_t(a^*) \le \theta_t(a) \vert \F_{t-1} = F_{t-1}}\\
            &\le \Pr\br*{B_{a,t}, \theta_t(a^*) \le y_a \vert \F_{t-1} = F_{t-1}}\\
            &\le (1-p_{a,t})\Pr\br*{B_{a,t}\vert \F_{t-1} = F_{t-1}}.
        \end{align*}
        Moreover, we have that
        \begin{align*}
            \Pr\br*{A_t = a^*, E_{a}^{\mu}\br*{t}, E_{a}^{\theta}\br*{t} \vert \F_{t-1} = F_{t-1}}
            &\geq \Pr\br*{\brc*{\theta_t(a) \le y_a}\cap\brc*{\forall b\in[K]: \theta_t(a^*) \ge \theta_t(b)} \vert \F_{t-1} = F_{t-1}}\\
            & \geq \Pr\br*{B_{a,t}, \theta_t(a^*) \geq \theta_t(a) \vert \F_{t-1} = F_{t-1}}\\
            &\ge \Pr\br*{B_{a,t}, \theta_t(a^*) > y_a \vert \F_{t-1} = F_{t-1}}\\
            &= p_{a,t}\Pr\br*{B_{a,t}\vert \F_{t-1} = F_{t-1}}.
        \end{align*}
        Rearranging and combining with the above yields
        {\small
        \begin{align*}
            \Pr\br*{A_t = a, E_{a}^{\mu}\br*{t}, E_{a}^{\theta}\br*{t} \vert \F_{t-1} = F_{t-1}}
            \le (1-p_{a,t})\Pr\br*{B_{a,t}\vert \F_{t-1} = F_{t-1}}
            \le \frac{1-p_{a,t}}{p_{a,t}}\Pr\br*{A_t = a^*, E_{a}^{\mu}\br*{t}, E_{a}^{\theta}\br*{t} \vert \F_{t-1} = F_{t-1}},
        \end{align*}}
        which shows that the inequality holds on histories where $E_{a}^{\mu}\br*{t}$ is true, concluding the proof.
    \end{proof}
    In order to bound the sum $S_3$, we bound $\E\brs*{q_{a,t}}$ and then combine with the above to get our final result. Denote:
    \begin{align*}
        \hat{\mu}_k^* \triangleq \frac{1}{k+1}\sum_{s=1}^{k}X_{a^*,s},
        \qquad \tilde{p}_{a,k} \triangleq \Pr\br*{\Ncal\br*{\hat{\mu}_k^*, \frac{1}{1+k}} > y_a \middle\vert X_{a^*,1}, \ldots, X_{a^*,k}},
        \qquad \tilde{q}_{a,k} \triangleq \frac{1 - \tilde{p}_{a,k}}{\tilde{p}_{a,k}}.
    \end{align*}
    We show the following:
    \begin{lemma}\label{lem:TSN_q}
        There exists an absolute constant $C_1 > 0$ such that
        \begin{align*}
            \sum_{k=0}^{\infty}\E\brs*{\tilde{q}_{a,k}} \le \frac{C_1}{\Delta_a^2}.
        \end{align*}
    \end{lemma}

    \begin{proof}
        Let $H_{a,k} \triangleq \brc*{\hat{\mu}_k^* \ge y_a + \frac{\Delta_a}{6}}$. We decompose $\tilde{q}_{a,k}$ as
        \begin{align*}
            \tilde{q}_{a,k} = \tilde{q}_{a,k}\Ind{H_{a,k}} + \tilde{q}_{a,k}\Ind{\overline{H_{a,k}}},
        \end{align*}
        and bound each part separately. We begin by looking at $\tilde{q}_{a,k}\Ind{H_{a,k}} = \frac{1 - \tilde{p}_{a,k}}{\tilde{p}_{a,k}}\Ind{H_{a,k}}$.
        Note that on $H_{a,k}$ we have $\hat{\mu}_k^* - y_a \ge \frac{\Delta_a}{6}$, which by \Cref{fact:Gaussian_normalized_tail_bound} implies that:
        \begin{align*}
            (1 - \tilde{p}_{a,k})\Ind{H_{a,k}} 
            &= \Pr\br*{\Ncal\br*{\hat{\mu}_k^*, \frac{1}{1+k}} \le y_a \middle\vert X_{a^*,1}, \ldots, X_{a^*,k}}\Ind{H_{a,k}}\\
            &= \Pr\br*{\Ncal(0,1) \le -\sqrt{k + 1}\br*{\hat{\mu}_k^* - y_a} \middle\vert X_{a^*,1}, \ldots, X_{a^*,k}}\Ind{H_{a,k}}\\
            &\le \Pr\br*{\Ncal(0,1) > \frac{\Delta_a\sqrt{k + 1}}{6} \middle\vert X_{a^*,1}, \ldots, X_{a^*,k}}\Ind{H_{a,k}}\\
            &\le \frac{1}{2}\exp\br*{-\frac{\Delta_a^2(k + 1)}{72}}.
        \end{align*}
        Let $x = \exp\br*{-\frac{\Delta_a^2(k + 1)}{72}}$. We showed that $(1 - \tilde{p}_{a,k})\Ind{H_{a,k}} \le \frac{x}{2}$,
        thus $\tilde{p}_{a,k}\Ind{H_{a,k}} \ge \br*{1 -\frac{x}{2}}\Ind{H_{a,k}}$, which gives us
        \begin{align*}
            \tilde{q}_{a,k}\Ind{H_{a,k}} = \frac{1 - \tilde{p}_{a,k}}{\tilde{p}_{a,k}}\Ind{H_{a,k}} \le \frac{\frac{x}{2}}{1 - \frac{x}{2}}\Ind{H_{a,k}}.
        \end{align*}
        We use the fact that $\frac{\frac{1}{2}x}{1 - \frac{1}{2}x} \le x$ for all $x \in \brs*{0, 1}$ to get
        \begin{align*}
            \tilde{q}_{a,k}\Ind{H_{a,k}} \le \exp\br*{-\frac{\Delta_a^2(k + 1)}{72}}\Ind{H_{a,k}},
        \end{align*}
        bounding $\tilde{q}_{a,k}$ under the event $H_{a,k}$. We now analyze the complementary summand. We show that for all $k$:
        \begin{align*}
            \Pr\br*{\overline{H_{a,k}}}\le 2\exp\br*{-\frac{k\Delta_a^2}{72}}.
        \end{align*}
        For $k = 0$, this bound is trivial. Hence assume $k \ge 1$. Using the definition of $\hat{\mu}_k^*$ and the fact that $\mu^* \le 1$, we have that
        {\small
        \begin{align*}
            \overline{H_{a,k}}
            &= \brc*{\hat{\mu}_k^* < \mu^* - \frac{\Delta_a}{6}}
            \subseteq \brc*{\frac{k}{k+1}\tilde{\mu}_k^* < \mu^* - \frac{\Delta_a}{6}}
            \subseteq \brc*{\tilde{\mu}_k^* < \mu^* + \frac{\mu^*}{k} - \frac{k+1}{k}\cdot \frac{\Delta_a}{6}}
            \overset{\mu^* \le 1}{\subseteq} \brc*{\tilde{\mu}_k^* < \mu^* - \frac{\Delta_a}{6} + \frac{1}{k}}
        \end{align*}}
        We look at two cases:
        
        For $k \ge \frac{12}{\Delta_a}$, we have that $\brc*{\tilde{\mu}_k^* < \mu^* - \frac{\Delta_a}{6} + \frac{1}{k}}
            \subseteq \brc*{\tilde{\mu}_k^* < \mu^* - \frac{\Delta_a}{12}}$, which by Hoeffding's inequality (\Cref{fact:hoeffding_inequality}) implies that
        \begin{align*}
            \Pr\br*{\overline{H_{a,k}}} 
            \le \Pr\br*{\tilde{\mu}_k^* < \mu^* - \frac{\Delta_a}{12}}
            \le \exp\br*{-\frac{k\Delta_a^2}{72}},
        \end{align*}
        Thus satisfying the required inequality.
        
        For $k < \frac{12}{\Delta_a}$, we have that $\frac{k\Delta_a^2}{72} <\frac{1}{6}$, thus yielding
        \begin{align*}
            \Pr\br*{\overline{H_{a,k}}} 
            \le 1 
            = e^{\frac{1}{6}}e^{-\frac{1}{6}} 
            \le e^{\frac{1}{6}}\exp\br*{-\frac{k\Delta_a^2}{72}},
        \end{align*}
        thus we have that $\Pr\br*{\overline{H_{a,k}}} \le 2 \exp\br*{-\frac{k\Delta_a^2}{72}}$ for all $k$. By Cauchy-Schwarz:
        \begin{align}\label{eq:TS_H_bad_event_CS}
            \E\brs*{\tilde{q}_{a,k}\Ind{\overline{H_{a,k}}}} 
            \le \sqrt{\E\brs*{\tilde{q}_{a,k}^2} \Pr\br*{\overline{H_{a,k}}}}
            \le \sqrt{2\E\brs*{\tilde{q}_{a,k}^2}}\exp\br*{-\frac{k\Delta_a^2}{144}}.
        \end{align}
        We now show that ${\E\brs*{\tilde{q}_{a,k}^2}}$ is bounded by an absolute constant for all $k$.
         Let $Y^k_i \vert \hat{\mu}_k^* \sim \Ncal\br*{\hat{\mu}_k^*, \frac{1}{1 + k}}$ be i.i.d. samples from the posterior of the optimal arm after $k$ draws. 
         Define $G_k  \vert \hat{\mu}_k^* \triangleq \inf\brc*{i \ge 1 : Y^k_i > y_a}$ to be the number of samples needed from $\Ncal\br*{\hat{\mu}_k^*, \frac{1}{1 + k}}$ to obtain a sample larger than $y_a$, noting that $G_k  \vert \hat{\mu}_k^* \sim Geo(\tilde{p}_{a,k})$. We show that there exists a constant $C_G$ such that
        \begin{align*}
            \sup_{k \ge 0} \E\brs*{G_k^2} \le C_G,
        \end{align*}
        since by the formula for the second moment of a geometric RV:
    \begin{align*}
        \E\brs*{G_k^2 \vert \hat{\mu}_k^*} = \frac{2 - \tilde{p}_{a,k}}{\tilde{p}_{a,k}^2} \overset{\tilde{p}_{a,k} \le 1}{\ge} \frac{1}{\tilde{p}_{a,k}^2},
    \end{align*}
    which implies that
    \begin{align*}
        \sup_{k\ge0}\E\brs*{\tilde{q}_{a,k}^2} 
        = \sup_{k\ge0}\E\brs*{{\frac{(1 - \tilde{p}_{a,k})^2}{\tilde{p}_{a,k}^2}}}
        \le \sup_{k\ge0}\E\brs*{{\frac{1}{\tilde{p}_{a,k}^2}}}
        \le \sup_{k\ge0}\E\brs*{\E\brs*{G_k^2 \middle\vert X_{a^*,1}, \ldots, X_{a^*,k}}}
        \le \sup_{k\ge0}\E\brs*{G_k^2}
        \le C_G.
    \end{align*}
    For $k = 0$, the result follows since $1 \ge \mu^* \ge y_a$, which implies that $\tilde{p}_{a,0} \ge \Pr\br*{\Ncal(0,1) > 1} > 0.1$, and so $\E\brs*{G_0^2} = \frac{2 - \tilde{p}_{a,0}}{\tilde{p}_{a,0}^2} < 200$ by the formula for the second moment of a geometric random variable. Let $k \ge 1$, let $r \ge 2$ and set $z_r \triangleq \sqrt{\frac{3}{2}\ln(r)}$. Since $y_a \le \mu^*$, we have that
    \begin{align*}
        \brc*{G_k \le r} \supseteq \brc*{\max_{j\in[r]}Y^k_j > \hat{\mu}_k^* + \frac{z_r}{\sqrt{k+1}}} \cap \brc*{{\hat{\mu}_k^* + \frac{z_r}{\sqrt{k+1}} \ge \mu^*}}.
    \end{align*}
    Taking the complement and applying the union bound, we get
        \begin{align*}
            \Pr\br*{G_k > r} 
            &\le \Pr\br*{\brc*{\max_{j\in[r]}Y^k_j \le \hat{\mu}_k^* + \frac{z_r}{\sqrt{k+1}}} \cup \brc*{\hat{\mu}_k^* + \frac{z_r}{\sqrt{k+1}} < \mu^*}}\\
            &\le \underbrace{\Pr\br*{\max_{j\in[r]}Y^k_j \le \hat{\mu}_k^* + \frac{z_r}{\sqrt{k+1}}}}_{(A)}
            + \underbrace{\Pr\br*{\hat{\mu}_k^* + \frac{z_r}{\sqrt{k+1}} < \mu^*}}_{(B)}.
        \end{align*}
        To bound $(A)$, we note that the sequence of samples $Y_i^k$ is i.i.d, giving us
        \begin{align*}
            \Pr\br*{\max_{j\in[r]}Y^k_j \le \hat{\mu}_k^* + \frac{z_r}{\sqrt{k+1}}} 
            = \br*{1 - \Pr\br*{\Ncal(0,1) > z_r}}^r
            \le \exp\br*{-r\Pr\br*{\Ncal(0,1) > z_r}}.
        \end{align*}
        Applying the lower bound on the tail of a Gaussian RV (\Cref{fact:Gaussian_normalized_tail_bound}) yields
        \begin{align*}
            \exp\br*{-r\Pr\br*{\Ncal(0,1) > z_r}}
            &\le \exp\br*{-r\frac{1}{\sqrt{2\pi}}\frac{z_r}{1 + z_r^2}e^{-z_r^2/2}}.
        \end{align*}
        Substituting back the definition of $z_r$
        \begin{align*}
            \exp\br*{-r\frac{1}{\sqrt{2\pi}}\frac{z_r}{1 + z_r^2}e^{-z_r^2/2}}
            = \exp\br*{-r\frac{1}{\sqrt{2\pi}}\frac{z_r}{1 + z_r^2}e^{-\br*{\sqrt{\frac{3}{2}\ln(r)}}^2/2}}
            \le \exp\br*{-r^{1/4}\frac{1}{\sqrt{2\pi}}\frac{z_r}{1 + z_r^2}}.
        \end{align*}
        Note that for $r \ge 2$, we have $z_r \ge 1$ and thus $\frac{z_r}{1 + z_r^2} \ge \frac{1}{2z_r}$, which gives us
        \begin{align*}
            \exp\br*{-r^{1/4}\frac{1}{\sqrt{2\pi}}\frac{z_r}{1 + z_r^2}}
            &\le  \exp\br*{-r^{1/4}\frac{1}{2\sqrt{2\pi}z_r}}
            \le  \exp\br*{-\frac{r^{1/4}}{2\sqrt{3\pi\ln(r)}}}.
        \end{align*}
        Note that $\ln(r)^{3/2} = o\br*{r^{1/4}}$, implying that there exists a constant $r_0 > 0$ such that for all $r \ge r_0$, we have that
        \begin{align*}
            \frac{r^{1/4}}{2\sqrt{3\pi\ln(r)}} \ge 3\ln(r).
        \end{align*}
        Thus, for all $r \ge \max\brc*{2,r_0}$:
        \begin{align*}
            \Pr\br*{\max_{j\in[r]}Y^k_j \le \hat{\mu}_k^* + \frac{z_r}{\sqrt{k+1}}}
            \le \exp\br*{-\frac{r^{1/4}}{2\sqrt{3\pi\ln(r)}}}
            \le \exp\br*{-3\ln(r)}
            \le \frac{1}{r^3}.
        \end{align*}
        To bound $(B)$, we let $\tilde{\mu}_k^* \triangleq \tilde{\mu}_{a^*, k}$, and note that $\hat{\mu}_k^* = \frac{k}{k+1}\tilde{\mu}_k^*$, and thus:
        \begin{align*}
        \brc*{\hat{\mu}_k^* + \frac{z_r}{\sqrt{k+1}} < \mu^*}
        \subseteq \brc*{\frac{k}{k+1}\tilde{\mu}_k^* < \mu^* - \frac{z_r}{\sqrt{k+1}}}
        \subseteq \brc*{\tilde{\mu}_k^* < \frac{k+1}{k}\mu^* - \frac{k+1}{k}\frac{z_r}{\sqrt{k+1}}}
        \subseteq \brc*{\tilde{\mu}_k^* < \mu^* - \frac{z_r - 1}{\sqrt{k}}}.
        \end{align*}
        Therefore by Hoeffding (\Cref{fact:hoeffding_inequality}):
        \begin{align*}
            \Pr\br*{\hat{\mu}_k^* + \frac{z_r}{\sqrt{k+1}} < \mu^*}
            \le \Pr\br*{\tilde{\mu}_k^* < \mu^* - \frac{z_r - 1}{\sqrt{k}}}
            \le \exp\br*{-2\br*{z_r - 1}^2}
            = \exp\br*{-2z_r^2 +4z_r -2}.
        \end{align*}
        Opening the exponent and substituting in the definition of $z_r$ yields
        \begin{align*}
            \exp\br*{-2\br*{z_r - 1}^2}
            &= \exp\br*{-2z_r^2 +4z_r -2}
            = \exp\br*{-2{{\frac{3}{2}\ln(r)}} + 4\sqrt{\frac{3}{2}\ln(r)} -2}
            = \frac{1}{r^3}\exp\br*{4\sqrt{\frac{3}{2}\ln(r)} -2}.
        \end{align*}
        Since $4\sqrt{\frac{3}{2}\ln(r)} - 2 = o\br*{\ln(r)}$, there exists $r_1 > 0$ such that for all $r > r_1$, we have that
        \begin{align*}
            4\sqrt{\frac{3}{2}\ln(r)} - 2 < \frac{1}{2}\ln(r).
        \end{align*}
        Thus, for all $r > r_1$:
        \begin{align*}
            \Pr\br*{\hat{\mu}_k^* + \frac{z_r}{\sqrt{k+1}} < \mu^*} 
            \le \frac{1}{r^3}\exp\br*{4\sqrt{\frac{3}{2}\ln(r)} -2}
            \le \frac{1}{r^{5/2}}.
        \end{align*}
        Combining, we get that for all $k$, and for all $r \ge R \triangleq \max\brc*{2, r_0, r_1}$:
        \begin{align}\label{eq:Bound_for_geometric_TS}
            \Pr\br*{G_k > r} \le \frac{1}{r^3} + \frac{1}{r^{5/2}} \le \frac{2}{r^{5/2}}.
        \end{align}
        Recalling that any $n \in \N$ can be written as $n^2 = \sum_{r=0}^{n-1}(2r+1)$, we have that for any RV $X$ with support in $\N$:      
       \begin{align*}
            \E\brs*{X^2} 
            = \sum_{n=1}^{\infty}n^2 \Pr\br*{X = n}
            &= \sum_{n=1}^{\infty}\Pr\br*{X = n}\sum_{r=0}^{n-1}(2r+1) \\
            &= \sum_{n=1}^{\infty}\Pr\br*{X = n}\sum_{r=0}^{\infty}(2r+1)\Ind{n > r} 
            =\E\brs*{\sum_{r=0}^{\infty}(2r+1)\Ind{X > r}}, 
        \end{align*}
        giving us that
        \begin{align*}
            \E\brs*{G_k^2} 
            = \E\brs*{\sum_{r=0}^{\infty}(2r+1)\Ind{G_k > r}}.
        \end{align*}
        Splitting the sum into $r \le R$ and $r > R$ and using \Cref{eq:Bound_for_geometric_TS} gives:
        \begin{align*}
            \E\brs*{\sum_{r=0}^{\infty}(2r+1)\Ind{G_k > r}}
            &= \E\brs*{\sum_{r=0}^{R}(2r+1)\Ind{G_k > r} + \sum_{r=R+1}^{\infty}(2r+1)\Ind{G_k > r}}\\
            &\le \sum_{r=0}^{R}(2r+1) + \E\brs*{\sum_{r=R+1}^{\infty}(2r+1)\Ind{G_k > r}}\\
            &\le (R + 1)^2 + \sum_{r=R+1}^{\infty}\frac{4r+2}{r^{5/2}}.
        \end{align*}
        Since $\sum_{r=R+1}^{\infty}\frac{4r+2}{r^{5/2}}$ is a convergent sum, we can let $C_G \triangleq \max\brc*{200, (R + 1)^2 + \sum_{r=R+1}^{\infty}\frac{4r+2}{r^{5/2}}} < \infty$ be a universal constant, since $R$ is a universal constant, and so by the above we have that for all $k\ge 0$:
        \begin{align*}
            \E\brs*{G_k^2} \le C_G.
        \end{align*}
        We recall that this implies that $\sup_{k\ge0}\E\brs*{\tilde{q}_{a,k}^2} \le C_G$, which by \Cref{eq:TS_H_bad_event_CS} implies that 
        \begin{align*}
            \E\brs*{\tilde{q}_{a,k}\Ind{\overline{H_{a,k}}}} 
            \le \sqrt{2\E\brs*{\tilde{q}_{a,k}^2}}\exp\br*{-\frac{k\Delta_a^2}{144}}
            \le \sqrt{2C_G}\exp\br*{-\frac{k\Delta_a^2}{144}}.
        \end{align*}
        We already showed that $\E\brs*{\tilde{q}_{a,k}\Ind{H_{a,k}}} \le \exp\br*{-\frac{\Delta_a^2(k + 1)}{72}}$, thus yielding
        \begin{align*}
            \sum_{k=0}^{\infty}\E\brs*{\tilde{q}_{a,k}}
            =\sum_{k=0}^{\infty}\br*{\E\brs*{\tilde{q}_{a,k}\Ind{H_{a,k}}} + \E\brs*{\tilde{q}_{a,k}\Ind{\overline{H_{a,k}}}}}
            \le \sum_{k=0}^{\infty}\br*{\exp\br*{-\frac{\Delta_a^2(k + 1)}{72}} + \sqrt{2C_G}\exp\br*{-\frac{k\Delta_a^2}{144}}}.
        \end{align*}
        Noting that $\exp\br*{-\frac{\Delta_a^2(k + 1)}{72}} \le \sqrt{2C_G}\exp\br*{-\frac{k\Delta_a^2}{144}}$ for all $k\ge0$ and applying the formula for the sum of an infinite geometric series, we get that
        \begin{align*}
            \sum_{k=0}^{\infty}\br*{\exp\br*{-\frac{\Delta_a^2(k + 1)}{72}} + \sqrt{2C_G}\exp\br*{-\frac{k\Delta_a^2}{144}}}
            \le 2\sqrt{2C_G}\sum_{k=0}^{\infty}\exp\br*{-\frac{k\Delta_a^2}{144}}
            \le \sqrt{8C_G}\frac{1}{1 - e^{-\frac{\Delta_a^2}{144}}}.
        \end{align*}
        Finally, using \Cref{fact:inequality_for_infinite_sums} and taking $C_1 = 145\sqrt{8C_G}$, we get that
        \begin{align*}
            \sum_{k=0}^{\infty}\E\brs*{\tilde{q}_{a,k}} 
            \le \sqrt{8C_G}\frac{1}{1 - e^{-\frac{\Delta_a^2}{144}}}
            \le \sqrt{8C_G}\br*{1 + \frac{144}{\Delta_a^2}}
            \le \frac{C_1}{\Delta_a^2},
        \end{align*}
        where the last inequality is since $\Delta_a^2 \le 1$.
    \end{proof}
    Combining the above, we get
    \begin{align*}
        &\sum_{t=1}^{\infty}\Pr\br*{\Pr\br*{A_t = a, E_{a}^{\mu}\br*{t}, E_{a}^{\theta}\br*{t} \vert \F_{t-1}} > \frac{\epsilon}{3M \Delta_a}}\\
        &\le \frac{3M \Delta_a}{\epsilon}\sum_{t=1}^{\infty}\Pr\br*{{A_t = a, E_{a}^{\mu}\br*{t}, E_{a}^{\theta}\br*{t}}}
        \tag{Markov's inequality}\\
        &\le \frac{3M \Delta_a}{\epsilon}\E\brs*{\sum_{t=1}^{\infty}q_{a,t}\Ind{A_t = a^*}\Ind{E_{a}^{\mu}\br*{t}}\Ind{E_{a}^{\theta}\br*{t}}}
        \tag{\cref{lem:TSN_prob_bound}}\\
        &\le \frac{3M \Delta_a}{\epsilon}\E\brs*{\sum_{t=1}^{\infty}q_{a,t}\Ind{A_t = a^*}}.
    \end{align*}
    Note that in round $t$ where arm $a^*$ is pulled for the $(k+1)^{th}$ time, we have $q_{a,t} = \tilde{q}_{a,k}$. Using this, we get that
    \begin{align*}
        \frac{3M \Delta_a}{\epsilon}\E\brs*{\sum_{t=1}^{\infty}q_{a,t}\Ind{A_t = a^*}}
        \le \frac{3M \Delta_a}{\epsilon}\E\brs*{\sum_{k=0}^{\infty}\tilde{q}_{a,k}}
        = \frac{3M \Delta_a}{\epsilon}\sum_{k=0}^{\infty}\E\brs*{\tilde{q}_{a,k}}
        \le \frac{3M \Delta_a}{\epsilon}\frac{C_1}{\Delta_a^2},
    \end{align*}
    where the last inequality is by \cref{lem:TSN_q}. Taking $C = 3C_1$ yields the desired inequality.
\end{proof}
We combine Lemmas \ref{lem:Bound_for_first_sum}, \ref{lem:Bound_for_second_sum} and \ref{lem:Bound_for_third_sum} to get
{\small
\begin{align*}
    &\sum_{t=1}^{\infty}\Pr \br*{\pi_t(a) > \frac{\epsilon}{M \Delta_a}}\\
    &\le {\sum_{t=1}^{\infty}\Pr\br*{\Pr\br*{A_t = a, E_{a}^{\mu}\br*{t}, E_{a}^{\theta}\br*{t} \vert \F_{t-1}} > \frac{\epsilon}{3M \Delta_a}}}
    + {\sum_{t=1}^{\infty}\Pr\br*{\Pr\br*{A_t = a, E_{a}^{\mu}\br*{t}, \overline{E_{a}^{\theta}\br*{t}} \vert \F_{t-1}} > \frac{\epsilon}{3M \Delta_a}}} \\
    &+ {\sum_{t=1}^{\infty}\Pr\br*{\Pr\br*{A_t = a, \overline{E_{a}^{\mu}\br*{t}} \vert \F_{t-1}} > \frac{\epsilon}{3M \Delta_a}}}\\
    &\le \frac{CM}{\epsilon\Delta_a} 
    + \frac{114 M}{\epsilon\Delta_a}\ln \br*{\frac{3M\Delta_a}{\epsilon}} 
    + \frac{33M}{2\epsilon\Delta_a}.
\end{align*}}
The above is true for all arms with $\Delta_a \ge \frac{\epsilon}{M}$, which in turn implies that
\begin{align*}
    \E\brs*{\sum_{t=1}^{\infty}{\Ind{\Delta_{\pi_t}>\epsilon}}}
    &= {\sum_{t=1}^{\infty}{\Pr \br*{\Delta_{\pi_t}>\epsilon}}}\\
    &\le \sum_{\Delta_a \ge \frac{\epsilon}{M}}{\sum_{t=1}^{\infty}\Pr \br*{\pi_t(a) > \frac{\epsilon}{M \Delta_a}}}\\
    &\le \sum_{\Delta_a \ge \frac{\epsilon}{M}}\br*{\frac{CM}{\epsilon\Delta_a} 
    + \frac{114 M}{\epsilon\Delta_a}\ln \br*{\frac{3M\Delta_a}{\epsilon}} 
    + \frac{33M}{2\epsilon\Delta_a}}.
\end{align*}
Defining $C_2 \triangleq 3\max\brc*{C, 114}$, and noting that for $\Delta_a \ge \frac{\epsilon}{M}$ we have that $\ln \br*{\frac{3M\Delta_a}{\epsilon}} \ge 1$, we get 
\begin{align*}
    \sum_{\Delta_a \ge \frac{\epsilon}{M}}\br*{\frac{CM}{\epsilon\Delta_a} 
    + \frac{114 M}{\epsilon\Delta_a}\ln \br*{\frac{3M\Delta_a}{\epsilon}} 
    + \frac{33M}{2\epsilon\Delta_a}}
    \le \sum_{\Delta_a \ge \frac{\epsilon}{M}}\frac{C_2M}{\epsilon\Delta_a}\ln \br*{\frac{3M\Delta_a}{\epsilon}}.
\end{align*}
Noting that $M \le K$, we get our desired instance-dependent upper bound. 
\end{proof}

\TSInstanceIndependent*
\begin{proof}
    Let $b > 0$, and consider the function $f(x) = \frac{1}{x} \ln \br*{bx}$ on the interval $(0, \infty)$. It holds that:
    \begin{align*}
        f'(x) = -\frac{\ln \br*{bx}}{x^2}+\frac{1}{x^2} = \frac{1}{x^2} \br*{1 - \ln \br*{bx}},
    \end{align*}
    Setting the derivative to $0$ yields $x^* = \frac{e}{b}$, and noting that $f''$ is negative at $x^*$ implies that $f$ is maximized at that point. 
    Applying this to our instance-dependent bound from \Cref{thm:thompson_sampling_is_U-ACE} with $b = \frac{3K}{\epsilon}$, we get
\begin{align*}
    \E\brs*{\sum_{t=1}^{\infty}{\Ind{\Delta_{\pi_t}>\epsilon}}}
    &\le \frac{CK}{\epsilon}\sum_{\Delta_a \ge \frac{\epsilon}{K}}\frac{1}{\Delta_a}\ln \br*{\frac{3K\Delta_a}{\epsilon}}\\
    &\le \frac{CK}{\epsilon}\sum_{\Delta_a \ge \frac{\epsilon}{K}}\frac{1}{\frac{\epsilon e}{3K}}\ln \br*{\frac{3K\frac{e\epsilon}{3K}}{\epsilon}}\\
    &\le \frac{3CK^3}{e\epsilon^2},
\end{align*}
Giving us the desired instance-independent bound.
\end{proof}

\subsection{Adding an optimal arm can increase the finite-horizon expected sample complexity}
\label{sec:optimal_arms_TS}
We use simulations to demonstrate that adding an additional optimal arm can strictly increase the $T$-expected sample complexity $\E\brs*{\sum_{t=1}^{T}\Ind{\Delta_{\pi_t} > \epsilon}}$.
This illustrates why, in contrast to standard regret analyses of Thompson sampling, 
we cannot assume that the quantity we analyze improves monotonically with the number of optimal arms.

Consider the two-armed Bernoulli instance
\begin{align*}
    \nu = \br*{Ber(1), Ber(0)},
\end{align*}
and the augmented three-armed instance
\begin{align*}
    \nu' = \br*{Ber(1), Ber(0), Ber(1)}.
\end{align*}
Thus, arm $1$ is the unique optimal arm in $\nu$, while arms $1$ and $3$
are optimal in $\nu'$. In both instances, arm $2$ is the only suboptimal
arm and satisfies $\Delta_2 = 1$.

Consider the policy $\pi^{TS\br*{\Ncal}}$ as defined in
\Cref{thm:thompson_sampling_is_U-ACE}. Since the rewards of all arms are
deterministic in both instances, the state of the algorithm at the beginning
of a round is completely determined by the number of times each arm has
been pulled. In particular, if an optimal arm $a$ has been pulled $n$ times,
then $S_t(a)=n$, while for the suboptimal arm $2$ we have $S_t(2)=0$.
Therefore, after $n$ pulls of arm $a$, its Gaussian sampling distribution is
\begin{align*}
    \theta(a)
    \sim
    \begin{cases}
        \Ncal\br*{\dfrac{n}{n+1},\dfrac{1}{n+1}},
        & a \text{ is optimal},\\[5mm]
        \Ncal\br*{0,\dfrac{1}{n+1}},
        & a=2.
    \end{cases}
\end{align*}

Since arm $2$ is the only suboptimal arm in both instances and
$\Delta_2=1$, the policy gap is simply the conditional probability of
selecting arm $2$:
\begin{align*}
    \Delta_{\pi_t}
    =
    \pi_t(2)
    =
    \Pr\br*{
        \theta_t(2) > \max_{a\neq 2}\theta_t(a)
        \,\middle|\,
        \F_{t-1}
    }.
\end{align*}

For the two-armed instance $\nu$, if the pull counts at the beginning of
round $t$ are $n_1,n_2$, the conditional probability of selecting the
suboptimal arm admits the closed form
\begin{align*}
    \pi_t(2)
    = \Phi\br*{
        -\frac{\frac{n_1}{n_1+1}}
        {\sqrt{\frac{1}{n_1+1}+\frac{1}{n_2+1}}}
    }.
\end{align*}
For the three-armed instance $\nu'$, if the pull counts are
$n_1,n_2,n_3$, the analogous probability is
\begin{align*}
    \pi_t(2)
    =
    \int_{-\infty}^{\infty}
        \phi(z)
        \Phi\br*{
            \sqrt{n_1+1}
            \br*{
                \frac{z}{\sqrt{n_2+1}}
                -\frac{n_1}{n_1+1}
            }
        }
        \Phi\br*{
            \sqrt{n_3+1}
            \br*{
                \frac{z}{\sqrt{n_2+1}}
                -\frac{n_3}{n_3+1}
            }
        }dz,
\end{align*}
where $\phi$ and $\Phi$ denote the standard Gaussian PDF and CDF,
respectively.
We simulated $100,000$ interactions with horizon $T = 100$ for Thompson sampling over the two instances. For the instance $\nu$, we calculated the policy gap via the closed formula, 
and for $\nu'$ we used high-order Gauss-Hermite quadrature to estimate the above integral. The results of this experiment are presented in \Cref{fig:TS_graph}.
\begin{figure}[ht]
    \centering
    \includegraphics[width=0.8\textwidth]{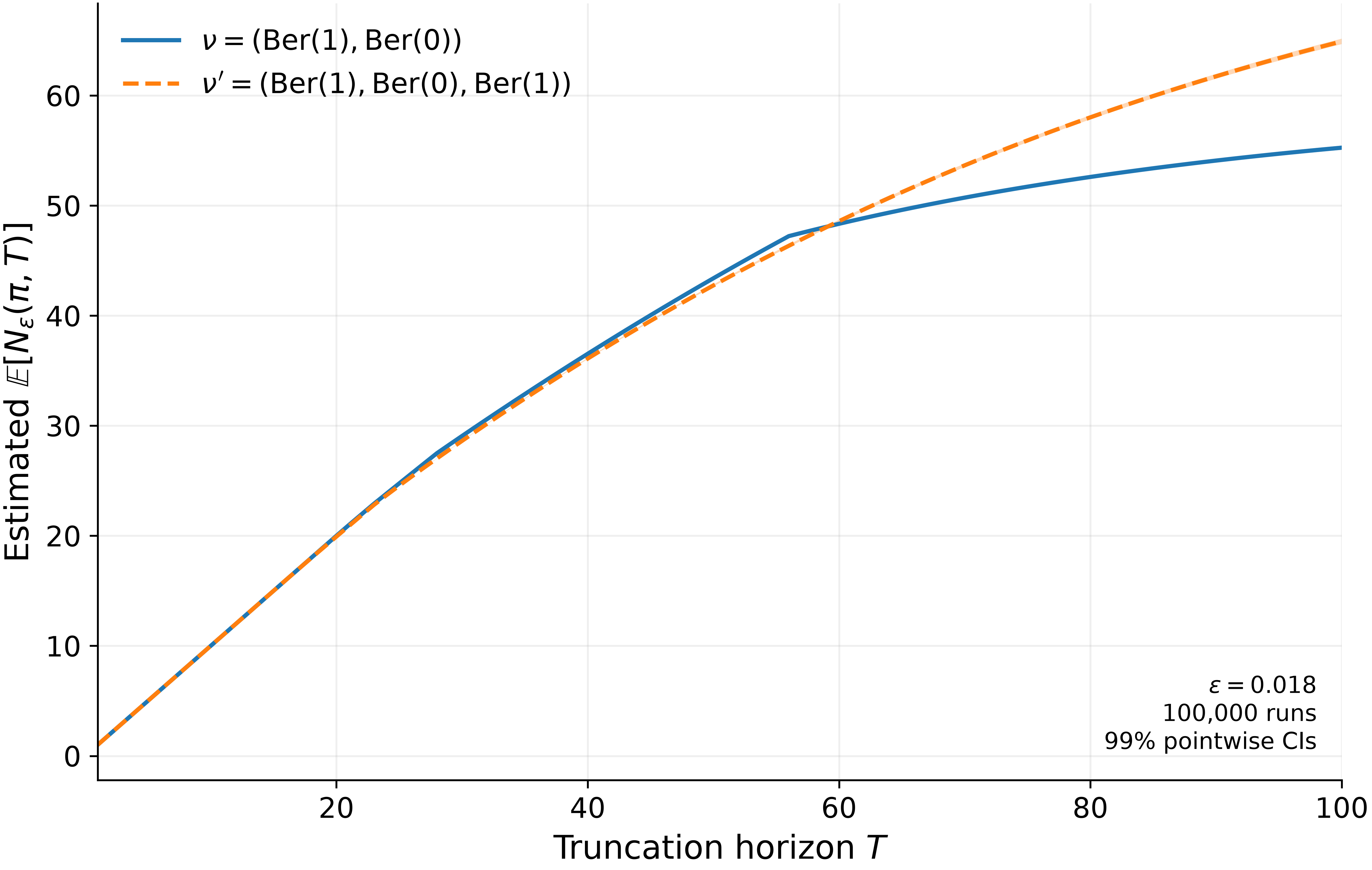}
    \caption{
    The X-axis is number of rounds, the Y-axis is the estimated expected sample complexity with $\epsilon = 0.018$ up to round $t$, using $100,000$ simulations of the interaction of $\pi^{TS(\Ncal)}$ with the environment. The 99\% CI for the expected sample complexity at round $100$ is $[55.076538, 55.452762]$ for $\E_{\nu}\brs*{N_{\epsilon}\br*{\pi^{TS(\Ncal)}, 100}}$, and $[64.676313, 65.170247]$ for $\E_{\nu'}\brs*{N_{\epsilon}\br*{\pi^{TS(\Ncal)}, 100}}$
    }
    \label{fig:TS_graph}
\end{figure}
We have empirically demonstrated that the addition of more optimal arms strictly increases the expected sample complexity after enough rounds.

Our intuition for the example is that although the additional optimal arm directly competes with the suboptimal arm, 
it also splits the observations between the two optimal arms. Consequently, the individual posteriors of the optimal arms concentrate more slowly, 
and this loss of concentration can outweigh the benefit of the additional optimal arm.

This example seeks to demonstrate that a point-wise monotonicity reduction argument, similar to that used in previous works that analyze TS \citep{AgrawalGoyalTS2017,LenientRegretMerlis2021}, is unavailable in our analysis, but nevertheless we overcame this hindrance in \Cref{lem:Bound_for_third_sum} using a careful choice of event that is agnostic to optimal arms, both in their number and in the arbitrary choice of optimal representative for the analysis.
\end{document}